\documentclass{article}

\PassOptionsToPackage{numbers}{natbib}

\usepackage[preprint]{neurips_2023}

\usepackage[utf8]{inputenc} %
\usepackage[T1]{fontenc}    %
\usepackage{hyperref}       %
\usepackage{url}            %
\usepackage{booktabs}       %
\usepackage{amsfonts}       %
\usepackage{nicefrac}       %
\usepackage{microtype}      %
\usepackage{xcolor}         %

\usepackage[parfill]{parskip}

\title{A Unified Model and Dimension \\
for Interactive Estimation}

\usepackage{times}

\newcommand{\relu}{\textit{relu}}

\usepackage[utf8]{inputenc} %
\usepackage[T1]{fontenc}    %
\usepackage{url}            %
\usepackage{booktabs}       %
\usepackage{amsfonts}       %
\usepackage{nicefrac}       %
\usepackage{microtype}      %
\usepackage{mathtools}
\usepackage{amsmath, amsthm, amssymb, natbib, graphicx, url, color}
\usepackage{algorithm, algorithmic, natbib}
\usepackage[scaled=.9]{helvet} %

\usepackage{color-edits}
\definecolor{cerulean}{rgb}{0.10, 0.58, 0.75}
\addauthor{miro}{cerulean}
\addauthor{rob}{magenta}

\usepackage{xspace}

\newcommand{\newdim}{dissimilarity\xspace}
\newcommand{\Newdim}{Dissimilarity\xspace}
\newcommand{\regret}{\text{\upshape Regret}}
\newcommand{\score}{\rho}

\newcommand{\scoreSQ}{\score_{\text{\upshape SQ}}}
\newcommand{\scoreSphere}{\score_{\text{\upshape sphere}}}
\newcommand{\scoreBandits}{\score_{\text{\upshape bandits}}}
\newcommand{\reward}{r}  %
\newcommand{\given}{\mathbin{\vert}}
\newcommand{\ball}{\mathcal{B}}
\newcommand{\sphere}{\mathcal{S}}
\DeclareMathOperator*{\argmin}{argmin}
\DeclareMathOperator*{\argmax}{argmax}

\usepackage{algorithm, algorithmic, natbib}
\usepackage[shortlabels]{enumitem}

\usepackage{thmtools}
\allowdisplaybreaks
\usepackage{bbm}
\usepackage[labelfont=it,margin=.5cm,]{caption}

\newtheorem{definition}{Definition}
\newtheorem{theorem}{Theorem}
\newtheorem{lemma}[theorem]{Lemma}

\newtheorem{example}[theorem]{Example}
\newtheorem{proposition}[theorem]{Proposition}
\newtheorem{corollary}[theorem]{Corollary}

\theoremstyle{remark}
\newtheorem*{claim*}{Claim}

\newcommand{\card}[1]{\lvert#1\rvert}
\newcommand{\Card}[1]{\left\lvert#1\right\rvert}
\newcommand{\abs}[1]{\lvert#1\rvert}

\newcommand{\braces}[1]{\{#1\}}

\newcommand{\BigBraces}[1]{\Bigl\{#1\Bigr\}}
\newcommand{\bigBracks}[1]{\bigl[#1\bigr]}
\newcommand{\BigBracks}[1]{\Bigl[#1\Bigr]}
\newcommand{\BiggBracks}[1]{\Biggl[#1\Biggr]}
\newcommand{\norm}[1]{\lVert #1 \rVert}
\newcommand{\BigParens}[1]{\Bigl(#1\Bigr)}
\newcommand{\bigParens}[1]{\bigl(#1\bigr)}
\newcommand{\biggParens}[1]{\biggl(#1\biggr)}
\newcommand{\bigGiven}{\mathbin{\bigm\vert}}
\newcommand{\BigGiven}{\mathbin{\Bigm\vert}}
\newcommand{\BigAbs}[1]{\Bigl\lvert#1\Bigr\rvert}
\newcommand{\bigAbs}[1]{\bigl\lvert#1\bigr\rvert}
\newcommand{\biggBraces}[1]{\biggl\{#1\biggr\}}
\newcommand{\bracesauto}[1]{\left\{{#1}\right\}}
\newcommand{\floor}[1]{\left\lfloor{#1}\right\rfloor}
\newcommand{\Ceil}[1]{\left\lceil#1\right\rceil}
\newcommand{\Parens}[1]{\left(#1\right)}
\newcommand{\set}[1]{\{#1\}}
\newcommand{\Set}[1]{\left\{#1\right\}}
\newcommand{\bigSet}[1]{\bigl\{#1\bigr\}}

\newcommand{\Alg}{{\textsf{\upshape Alg}}}
\newcommand{\one}{\mathbf{1}}
\DeclareMathOperator{\rank}{rank}
\newcommand{\AlgOLS}{\mathcal{R}\textit{eg}}
\newcommand{\RegretOLS}[1]{\mathrm{Regret}_{\mathcal{R}\textit{eg}}(#1)}

\def\X{{\mathcal X}}
\def\O{{\mathcal O}}

\def\H{{\mathcal H}}

\def\A{{\mathcal A}}
\def\B{{\mathcal B}}

\def\F{{\mathcal F}}
\def\Z{{\mathcal Z}}

\def\E{{\mathbb E}}
\def \Mb{{\mathbf M}}

\def \Ub{{\mathbf U}}

\newcommand{\Db}{{\mathbf D}}
\newcommand{\Sb}{{\mathbf S}}
\newcommand{\Ab}{{\mathbf A}}
\newcommand{\Kb}{{\mathbf K}}

\newcommand{\C}{\mathcal{C}}
\newcommand{\reals}{{\mathbb R}}

\def\d{\mathtt{d}}
\def\dmon{\overline{\d}}
\newcommand{\dimSQ}{\mathtt{dim}_\text{\upshape SQ}}
\newcommand{\dimE}{\mathtt{dim}_\text{\upshape E}}
\newcommand{\dimEmon}{\overline{\mathtt{dim}}_\text{\upshape E}}

\newcommand{\dsq}{d_\text{\upshape SQ}}
\newcommand{\drho}{d_{\rho}}

\newcommand{\ab}{\mathbf{a}}

\newcommand{\ub}{\mathbf{u}}

\newcommand{\vb}{\mathbf{v}}
\newcommand{\zb}{\mathbf{z}}
\newcommand{\zhat}{\hat{z}}
\newcommand{\thetab}{\boldsymbol{\theta}}

\author{%
  Nataly Brukhim \\
  Princeton University\\
  \texttt{nbrukhim@princeton.edu} \\
  \And
  Miroslav Dud{\'i}k \\
  Microsoft Research\\
  \texttt{mdudik@microsoft.com} \\
  \And
  Aldo Pacchiano \\
  Microsoft Research\\
  \texttt{apacchiano@microsoft.com} \\
  \And
  Robert Schapire \\
  Microsoft Research\\
  \texttt{schapire@microsoft.com} \\
}

\begin{document}

\maketitle

\begin{abstract}

We study an abstract framework for interactive learning called
\emph{interactive estimation} in which
the goal is to estimate a target from its ``similarity'' to points queried by the learner.
We introduce a combinatorial measure called \emph{\newdim dimension}
which largely captures learnability in our model.
We present a simple, general, and broadly-applicable algorithm, for which we obtain both regret and PAC generalization bounds that are polynomial in the new dimension.
We show that our framework subsumes and thereby unifies
two classic learning models:
statistical-query learning and structured bandits.
We also delineate how the \newdim dimension is related to well-known parameters for both frameworks, in some cases yielding significantly improved analyses.

\end{abstract}

\section{Introduction}

We study a general interactive learning protocol called \emph{interactive estimation}.
In this model, the learner repeatedly queries the environment with an element from a set of \emph{alternatives},
and observes a stochastic reward whose expectation is given by an arbitrary measure of the ``similarity'' between the queried alternative and the unknown ground truth.
Thus, in rough terms,
the goal is to estimate a target from its similarity to queried alternatives.
By studying such a general abstraction of interactive learning,
we are able to reason about the properties of a very broad family of learning settings, and to make connections across a variety of contexts.

Our results are based on a combinatorial complexity measure we introduce called the \emph{\newdim dimension}, which we show largely captures learnability in our model. Intuitively, this measure corresponds to the length of the longest sequence of alternatives
in which each one has a similar suboptimal value of similarity to all its predecessors. We then use the measure to analyze the performance of a simple, broadly-applicable class of algorithms which repeatedly make new queries that best fit the preceding observations.  We prove both regret bounds and PAC generalization bounds that are all polynomial in the \newdim dimension.

We show that our learning framework subsumes two classic learning models that were seemingly unrelated prior to this work:

First, our model subsumes
the statistical query (SQ) model, introduced by \citet{kearns1998efficient} for designing noise-tolerant learning algorithms.
In the SQ model, the learner can sequentially ask certain queries of an oracle, who responds with answers that are only approximately correct, with the goal of correctly estimating a target. Despite its simplicity, it has been proven to be a powerful model. Indeed, a wide range of algorithmic techniques in machine learning are implementable using SQ learning. Thus, it has been
proven useful, not only for designing noise-tolerant algorithms, but also for its connections to other noise
models,  and as an explanatory tool to prove
hardness of many problems
(see the survey of \citet{reyzin2020statistical}).
We show that our framework subsumes the SQ model, and furthermore that the \newdim dimension generalizes well-known parameters that characterize SQ learnability.

Second, our model captures
structured bandits, in which the learner repeatedly chooses actions
which yield stochastic rewards, with the goal of minimizing regret relative to the best action in hindsight. Over more than a decade, the eluder
dimension \citep{russo2013eluder} has been a central technique for analyzing regret for contextual
bandits and reinforcement learning (RL) with function approximation \citep{wen2013efficient,osband2014model,  wang2020reinforcement, dong2021provable}.
We will see that
the \newdim dimension is upper-bounded by the eluder dimension, and that there can in fact be a large gap between the two. This sometimes leads to an improved analysis when relying on the proposed \newdim measure rather than the eluder dimension.

Because SQ and bandits are both subsumed by our framework, all the results mentioned above directly apply to those settings as well,
including the applicability of our general-purpose algorithms.

To summarize, our main contributions are as follows:
\begin{itemize}
    \item \textbf{Unified framework.} We derive a general framework which captures various interactive learning settings, including specifically SQ and bandits.

\item \textbf{Novel dimension, performance bounds.} We introduce the \newdim dimension that largely characterizes learnability in our framework. We study a general, simple algorithm, and give a novel analysis that results in both regret and PAC generalization bounds that are polynomial in the new
dimension. We also give lower bounds in the SQ and bandit settings.
 \item \textbf{Improved analysis.}
We show instances in which the standard analysis of a certain class of algorithms using the eluder dimension yields bounds that are arbitrarily large, but in which an analysis using our dimension yields low regret bounds.
\end{itemize}

\textbf{Related work.\;\;}
The interactive estimation model we consider in this work is defined with respect to an evaluation function that can be thought of as an arbitrary measure of the “similarity” between
the queried alternative and the target.
Previously, \citet{balcan2006theory} developed a theory
of similarity-based learning that generalizes kernel methods,
providing sufficient conditions for a similarity function to be useful for learning.
\citet{chen2009similarity} review several approaches to classification  based on similarity between examples, including, for instance, kernels and nearest neighbors.
\citet{ben1995learning} studied a learning-by-distances model that resembles ours using a metric as a measure of similarity.
In comparison to these works,
our model admits an arbitrary similarity measure for which we derive a general dimension, algorithm, and bounds.

In the context of bandit and reinforcement learning, a parameter called the decision-estimation coefficient (DEC) has recently been proposed by \citet{foster2021statistical}
to characterize learnability in interactive decision making. Unlike DEC, our dimension is combinatorial in nature, and applies to settings like SQ, which are not captured by DEC.

As discussed above, our model subsumes SQ and bandits,
both of which have been extensively studied
(see the references above as well as various surveys \citep{lattimore2020bandit,reyzin2020statistical}).

\section{Setting} \label{sec:setting}

In this paper, we study an interactive learning protocol called \emph{interactive estimation}. In this protocol, the learner is provided with a set $\Z$ of \emph{alternatives}, and an \emph{evaluation function} $\score:\Z\times\Z\to [-1,1]$.
Intuitively, $\score$ can be viewed as a measure of ``similarity,'' though it need not be symmetric.
There is also a distinguished alternative $z^*\in\Z$ called the \emph{target}, fixed throughout the interaction, and unknown to the learner.
In each of a sequence of steps $t=1,\ldots,T$, the learner selects
one alternative $z_t\in\Z$ and receives a stochastic \emph{reward} $r_t\in [-1,1]$ drawn independently, conditioned on $z_t$,
with expectation satisfying $\E[r_t\given z_t]=\score(z_t\given z^*)$.
Informally,
by choosing alternatives and observing their similarity to $z^*$, the learner aims to get close to the target.
The special case when $\reward_t=\score(z_t\given z^*)$,
that is, when rewards are deterministic functions of the queried alternatives, is referred to as the \emph{deterministic setting}.

We generally assume $\score(z^*\given z^*)\geq \score(z\given z^*)$ for all $z\in\Z$ and denote this optimal value as $\alpha^*\coloneqq\score(z^*\given z^*)$. We will assume that the value of $\alpha^*$ is known to the learner or that we are provided with an alternate \emph{optimality level} $\alpha\le\alpha^*$ such that the task is to identify $z$ with $\score(z\given z^*)\ge\alpha$. At the end of Section~\ref{sec:upper} we discuss how this assumption can be relaxed.

We consider two alternative goals for a learner in this model: \emph{sublinear regret} and \emph{PAC generalization}.
A learner achieves \emph{sublinear regret} relative to an optimality level $\alpha\le\alpha^*$ if
$\regret(T,\alpha) = o(T)$, where
\[
\regret(T,\alpha) =  \smash[t]{\sum_{t=1}^T}
\BigParens{
  \alpha - \score(z_t\given z^*)
}.
\]
We say that a learner achieves \emph{PAC generalization} if
for any $\epsilon, \delta > 0$ and $\alpha\le\alpha^*$,
with probability at least $1-\delta$ (over the randomness of the query responses and the learner's own randomization),
after $m(\epsilon,\delta,\alpha)$ interactions in the protocol above, the learner outputs $\zhat$ such that
$\score(\zhat\given z^*) \ge  \alpha - \epsilon$. The function $m(\epsilon,\delta,\alpha)$ is referred to as sample complexity.
We recover standard notions of regret and PAC generalization by setting $\alpha=\alpha^*$.

\begin{example}[Point on a sphere]
\label{ex:sphere}
Let $\norm{\cdot}$ denote the standard Euclidean norm in $\reals^n$ and let
$\Z=\sphere_{n-1}=\braces{\zb\in\reals^n:\:\norm{\zb}=1}$ be the unit sphere in $\reals^n$.
The goal is to estimate an unknown point $\zb^*\in\sphere_{n-1}$ based on rewards equal to the inner product between the queries and the target, that is, $r_t=\scoreSphere(\zb_t\given\zb^*)\coloneqq
\langle \zb_t , \zb^*\rangle$.
\end{example}

 We now introduce the two main examples corresponding to classic learning models that are subsumed by the interactive estimation model.

\begin{example}[Structured bandits]
\label{ex:bandits}
Let $\A$ be an action set, $\F$ a space of reward functions $f : \A \rightarrow[-1,1]$, and ${f^*\in\F}$ the target reward function.
In step $t$, the learner chooses an action $a_t\in\A$ and receives reward $r_t\in[-1,1]$ with $\E[r_t\given a_t]=f^*(a_t)$. The goal is to maximize the sum of rewards.
Let $a^* = \argmax_{a \in \A} f^*(a)$ be an optimal action.
To represent bandits in our formalism, we let $\Z = \F \times \A$,
$z^* = (f^*, a^*)$,  and
$\scoreBandits\bigParens{(f,a)\bigGiven(f^*, a^*)}  = f^*(a) $.
\end{example}

The structured bandit problem has been extensively studied, and Example \ref{ex:bandits} captures its expressiveness within our framework. For example, it recovers the possibly simplest case of $K$-armed bandits, by considering $\A = \{1,...,K\}$ and $\F = [0,1]^K$. At each round, the learner chooses arm $a_t \in \A$ and observes a reward $r_t$ which is drawn from a distribution with mean $f^*(a_t)$.
See Appendix \ref{appendix:mab} for a more concrete example of $K$-armed bandits instantiated within our framework.
In Section \ref{sec:bandits} we also give concrete bounds for other example classes including linear bandits and GLM bandits.

\begin{example}[SQ learning] \label{ex:SQ}
Given a domain~$\X$, the goal is to learn a binary
classifier $h^*:\X\to\braces{\pm 1}$ from some hypothesis
class $\H\subseteq\{\pm 1\}^\X$, based on training examples $(x,y)$ drawn from some distribution $D$ such that $y=h^*(x)$. In step $t$, the learner produces a hypothesis
$h_t$ and observes the accuracy of $h_t$ on a fresh finite sample. In this case, $\Z=\H$, the evaluation function is equal to the expected accuracy $\scoreSQ(h\given h^*)= \E_{x\sim D}[h(x)h^*(x)] $, and the reward is the empirical accuracy on a fresh sample.\looseness=-1
\end{example}

The SQ learning model considered in this work (Example~\ref{ex:SQ})  differs from the original model {of~\citet{kearns1998efficient}}
because it is restricted, as in previous works~\cite{bshouty2002using, feldman2017statistical, yang2005new}, to so-called \emph{correlational} queries (called CSQs) and assumes stochastic responses, as opposed to allowing arbitrary queries and adversarial responses.
We discuss relationships between various SQ variants in Appendix~\ref{appendix:sq}.

We finish this section by introducing a central concept of this paper, a new combinatorial complexity measure called the \emph{\newdim dimension} which, as we will see, largely captures learnability in the interactive estimation protocol.

\begin{definition}[\Newdim dimension] \label{def:new_dim}
For a set $\Z$, scalars $\alpha\in\reals$, $\epsilon > 0$,  and evaluation function $\score:\Z\times\Z\to [-1,1]$, the \emph{\newdim dimension} $\d_\score(\Z, \alpha, \epsilon)$  is the largest integer $d$ for which there exist $z_1,\dotsc, z_d \in \Z$ with  $\score(z_i\given z_i ) \ge \alpha$, and a scalar $c \le \alpha -  \epsilon$, such that for all $i<j$,
\[
\BigAbs{\score(z_{i}\given z_j) - c} \le \frac{\epsilon}{\sqrt{d}}.
\]
Furthermore, denote the \emph{monotonic \newdim dimension} as $\dmon_\score(\Z, \alpha, \epsilon)\coloneqq \max_{\epsilon' \ge \epsilon} \d_\score(\Z, \alpha, \epsilon')$.
\end{definition}

Note that $\d_\score(\Z, \alpha, \epsilon)=0$ if there is no $z$ such that $\score(z\given z)\ge\alpha$, and otherwise $\d_\score(\Z, \alpha, \epsilon)\ge 1$. In particular, if $\alpha\le\alpha^*$ then $\d_\score(\Z, \alpha, \epsilon)\ge 1$.

In rough terms, this dimension corresponds to the longest sequence of points with $\alpha$-large self-evaluation, such that the evaluation $\score(z_i\given z_j)$ of each point $z_i$ relative to every {successive point $z_j$} is ``small'' (significantly less than $\alpha$), and also tightly clustered around some value $c$.
Thus, each point is similar to itself, but dissimilar {from all successive points to about the same degree}.
The idea is illustrated in Figure~\ref{fig:dim}. {The monotonic dissimilarity
dimension is the tightest upper bound on the dissimilarity dimension that is non-increasing in $\epsilon$.}

\begin{figure}
\centering
\includegraphics[width = 13cm]{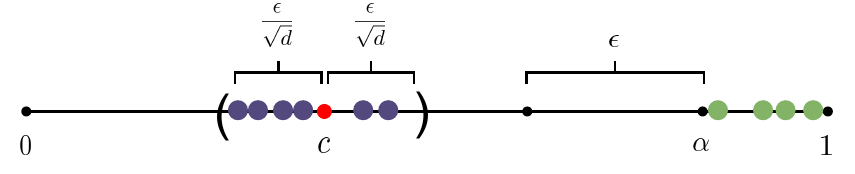}
\caption{An illustration of the \newdim dimension $\d_\rho(\Z, \alpha, \epsilon) = 4$. The figure shows $\rho$ values on the $[0,1]$ interval for a sequence of elements $z_1,z_2,z_3,z_4$. The four green points represent self-evaluation values $\rho(z_i\given z_i)$ for $i=1,\dotsc,4$; all are greater than or equal to $\alpha$. The six blue points represent values $\rho(z_i\given z_j)$ for $i < j$; all are within the distance $\epsilon/\sqrt{d}$ from the value $c\le\alpha - \epsilon$.}
\label{fig:dim}
\end{figure}

Various concrete examples where the \newdim dimension can be bounded are provided in %
Section~\ref{sec:bandits}.
For instance, using a general bound for linear bandits from Section~\ref{sec:bandits}, we can show that for the task of finding a point on a sphere based on inner products (Example~\ref{ex:sphere}), the dimension $\dmon_{\scoreSphere}(\Z, \alpha, \epsilon) \le 4n + 3$, a bound that is independent of both $\alpha$ and $\epsilon$.

\section{Algorithms and upper bounds} \label{sec:upper}

In this section we analyze algorithms for the interactive estimation protocol, which we call \emph{interactive estimation algorithms}.
We show that when an interactive estimation algorithm satisfies
two properties, \emph{large self-evaluations} and \emph{decaying estimation error}, then its regret can be bounded using the \newdim dimension. We introduce a simple algorithm (Algorithm~\ref{alg:MAIN1}), which satisfies these properties for many standard classes of alternatives.
The first property requires that the algorithm only select alternatives that would achieve the expected reward of at least $\alpha$ if they were the target:

\begin{definition}[$\alpha$-large self-evaluations]
    An interactive estimation algorithm has $\alpha$-large self-evaluations if at every time step $t=1,\dotsc,T$, it selects a query $z_t$ such that $z_t \in \Z_\alpha$, where
\begin{equation}
    \Z_\alpha = \{ z \in \Z : \score(z\given z) \ge \alpha\}.
\end{equation}
\end{definition}

Algorithm~\ref{alg:MAIN1} satisfies this property, with the optimality level $\alpha$ provided as input. At the end of the section, we discuss the case when $\alpha$ is not provided and $\alpha^*$ is unknown.  We derive an optimistic version of Algorithm~\ref{alg:MAIN1} that achieves $\alpha^*$-large self-evaluations with high probability.

The second property states that the queries produced by the algorithm provide increasingly good estimates %
of the expected rewards in the previous rounds (that is they are good estimators
\emph{with the benefit of hindsight}), as quantified by the square loss.

\begin{definition}[Decaying estimation error]\label{assume:est:error}
An interactive estimation algorithm has decaying estimation error if there exists $C_{T,\delta} \ge 0$ growing sublinearly in~$T$, that is, $C_{T,\delta}=o(T)$, such that with probability at least $1-\delta$ the sequence of queries $z_1,\dotsc,z_T$ produced by the algorithm satisfies
\begin{equation}
\label{equation::least_squares_main_guarantee}
\sum_{i=1}^{t-1} \BigParens{ \score(z_i\given z_t) - \score(z_i\given z^*) }^2 \leq C_{T,\delta}
\end{equation}
 for all $t\in\{1,\dotsc,T\}$ simultaneously.
\end{definition}

\begin{algorithm}[t]
    \caption{Interactive Estimation via Least Squares}\label{alg:MAIN1}
    \begin{algorithmic}[1]
        \STATE \textbf{Input:} {set of alternatives} $\Z$, evaluation function $\rho$,
               optimality level $\alpha$, number of steps $T$.
        \FOR{$t=1,\dotsc,T$}
        \STATE Submit the query
               $\displaystyle
                z_t = \argmin_{z\in\Z_\alpha}
                      \sum_{i=1}^{t-1}
                      \BigParens{\score(z_i\given z) - \reward_i}^2$.
               \label{step:zj}
        \STATE Observe reward $\reward_t$.
        \ENDFOR
    \end{algorithmic}
\end{algorithm}

Algorithm~\ref{alg:MAIN1} optimizes an empirical version of
Eq.~\eqref{equation::least_squares_main_guarantee}, with the observed rewards $r_i$ in place of the expectations $\score(z_i\given z^*)$. Thus, in the deterministic setting, with $\reward_i=\score(z_i\given z^*)$, Algorithm~\ref{alg:MAIN1} satisfies this property with $C_{T,\delta}=0$. It can also be shown that it satisfies this property when the set of alternatives is finite:

\begin{theorem}\label{thm:C_T_delta_finite_Z}
Assume that $|\Z| < \infty$.
Then Algorithm \ref{alg:MAIN1} satisfies the decaying estimation error property with
 $
 C_{T,\delta} = O\bigParens{\ln(2T|\Z|/\delta)}
 $.
\end{theorem}

In the general case, when $\Z$ is infinite, we show that Algorithm \ref{alg:MAIN1} satisfies the decaying estimation error property with $C_{T,\delta} = O\bigParens{\log(TN/\delta)}$ where $N$ is a suitable covering number of $\Z$ (see Corollary~\ref{cor:alg:MAIN1} in Appendix~\ref{section::least_squares}).  For example, for \emph{linear bandits}, which is an instance of Example \ref{ex:bandits} in which the action set and function class correspond to a subset of the unit ball in $\mathbb{R}^n$, we obtain $C_{T,\delta} = O\bigParens{n\log(1/\epsilon) + \log(T/\delta)}$.

In Appendix~\ref{appendix::regression_oracle}, we discuss an approach in which we have access to an online regression oracle for the least squares problem in step~\ref{step:zj}. We show that a suitably modified version of Algorithm \ref{alg:MAIN1} has a decaying estimation error as long as the online regression oracle achieves a sublinear regret
(but without further dependence on a covering number).

To develop some intuition how Algorithm \ref{alg:MAIN1} works,
we can again consider the $K$-armed bandit problem (a special case of Example \ref{ex:bandits}), and suppose that $\alpha=0.75$. In each step, the algorithm picks a pair $(f_t,a_t)$, where $f_t\in[0,1]^K$ is the vector of mean reward estimates and $a_t$ is the arm with the largest mean estimate. The estimates $f_t(a)$, $a=1,\dotsc,K$, are formed by optimizing the least squares error of the observed rewards, under the constraint that at least one of the mean estimates must be above $0.75$. As a result, the algorithm pulls the arm with the largest average reward as long as that average is above $0.75$ (arms that have not been pulled are assumed to have averages above $0.75$). If all the averages are below $0.75$ then the algorithm selects the arm $a$ with the smallest value $n_a(0.75-\hat{\mu}_a)^2$, where $n_a$ is how many times the arm has been pulled so far and $\hat{\mu}_a$ is its average reward; it can be verified that this solves the least squares problem subject to the constraint that at least one of the mean estimates is above $0.75$.\looseness=-1

We next state our main results: a regret bound and a PAC generalization guarantee. They are both based on bounding how many ``bad'' queries
any algorithm with large self-evaluations and decaying estimation error can make. Concretely, we say that a query $z \in \Z$ is $\epsilon$-\textit{bad} if its
 suboptimalty gap is greater than $\epsilon$, that is, if
\[
     \score(z\given z^*) < \alpha - \epsilon.
\]
The next lemma shows that the number of $\epsilon$-bad queries is upper bounded polynomially in the \newdim dimension.

\begin{lemma}[Few bad queries] \label{lem:few_bad_queries}
  Let $\epsilon,\delta>0$, and let $d= \dmon_\score(\Z, \alpha, \epsilon) < \infty$
  for some set~$\Z$, evaluation function $\score$ and $\alpha\le\alpha^*$.
Let $\Alg$ be an interactive estimation algorithm with $\alpha$-large self-evaluations and a decaying
estimation error with some $C_{T,\delta}$. Then, with probability at least $1-\delta$, the number of $\epsilon$-bad queries that $\Alg$ makes in $T$ steps is
at most
$2d^{1.5}\ln(4/\epsilon) + 12d^{2.5}C_{T,\delta}/\epsilon^2$.
Consequently, if $C_{T,\delta}\ge\ln(2T)$, then with probability at least $1-\delta$, the number of $\epsilon$-bad queries is at most $36d^{2.5}C_{T,\delta}/\epsilon^2$, and if $C_{T,\delta}=0$ then it
is at most $2d^{1.5}\ln(4/\epsilon)$.
\end{lemma}

The above result is the core component of our main theorems. The proof is given in Appendix~\ref{appendix:few_bad_queries_pf}; here we sketch the main ideas. The goal is to show that the  ``bad''
interval $[-1, \alpha - \epsilon]$ cannot contain too many queries made by $\Alg$. The proof starts by partitioning this interval
into disjoint subintervals and then bounds the number of queries in each subinterval. It does so by constructing a graph with nodes corresponding to queries, which are connected by an edge if they satisfy the dimension conditions. The decaying errors that imply a certain minimum number of edges (as a function of number of queries). On the other hand, the dissimilarity dimension bounds the size of the largest clique, which implies an upper bound on the number of edges (using Tur\'an’s Theorem~\citep{turaan1941extremal}, a standard result from extremal graph theory). Combining the bounds yields an upper bound on the number of queries in the subinterval. Summing across subintervals proves the lemma.\looseness=-1

The following theorems use Lemma \ref{lem:few_bad_queries} to bound both the regret and PAC sample complexity. The proofs are deferred to Appendices~\ref{appendix:pf_main_regret}
and~\ref{appendix:pf_main_pac}.

\begin{theorem}[Regret]\label{thm:regret_optimism}
  Let $\delta, T > 0$,
  and let $d= \dmon_\score(\Z, \alpha, 1/T)$
  for some set $\Z$, evaluation function~$\score$ and $\alpha\le\alpha^*$.
Let $\Alg$ be an interactive estimation algorithm with $\alpha$-large self-evaluations and a decaying
estimation error with some $C_{T,\delta}$. If $C_{T,\delta}\ge\ln(2T)$ then with probability at least $1-\delta$, the regret of $\Alg$ satisfies
\[
\mathrm{Regret}(T,\alpha)\le 1+ 12d^{1.25}\sqrt{C_{T,\delta} T}.
\]
In the deterministic setting,
$
\mathrm{Regret}(T,\alpha)\le 1+12d^{1.5}
$.
\end{theorem}

For an algorithm with a decaying estimation error, the term  $C_{T,\delta}$ is sublinear in $T$, implying a sublinear regret in Theorem~\ref{thm:regret_optimism}. For example,  Algorithm~\ref{alg:MAIN1} has a decaying estimation error with $C_{T,\delta}$ that scales logarithmically with
$T/\delta$ for many standard function classes, and so the overall regret scales as $O(\sqrt{T\log T})$
(see Corollary~\ref{cor:alg:MAIN1} in Appendix~\ref{section::least_squares}).\looseness=-1

\begin{algorithm}[t]
    \caption{PAC Interactive Estimation}\label{alg:MAIN12}
    \begin{algorithmic}[1]
        \STATE
        \textbf{Input:}
          {set of alternatives} $\Z$,
          {evaluation} function $\rho$,
          optimality level $\alpha$,\\
        \hphantom{\textbf{Input:}}
          base interactive estimation algorithm $\Alg$, parameters $T$, $n_1$, $n_2$.
\smallskip
        \STATE Run algorithm $\Alg$ with the provided $\Z, \rho, \alpha, T$. \label{step:run_base}
        \STATE Sample $n_1$ indices $t_1,\dotsc,t_{n_1}$ uniformly at random from $\{1,\dots,T\}$.
        \STATE For each $\ell=1,\dotsc,n_1$, submit query $z_{t_{\ell}}$ for $n_2$ times; denote the average response $\bar{\reward}_{t_{\ell}}$.
        \STATE Let $\smash{\hat{\ell}}=\argmax_{\ell\in\{1,\dotsc,n_1\}} \bar{\reward}_{t_{\ell}}$.
        \STATE
        \textbf{Output} $\zhat = z_{t_{\hat{\ell}}}$.
    \end{algorithmic}
\end{algorithm}

To derive PAC generalization guarantees, we apply a variant of online-to-batch reduction to any algorithm with large self-evaluations and a decaying estimation error. The resulting approach, shown in Algorithm~\ref{alg:MAIN12}, satisfies the following guarantee (proved in Appendix~\ref{appendix:pf_main_pac}):

\begin{theorem}[PAC generalization]\label{thm:main_pac}
  Let $\epsilon,\delta>0$, and let $d= \dmon_\score(\Z, \alpha, \epsilon)$
  for some set $\Z$, evaluation function $\score$ and $\alpha\le\alpha^*$.
Let $\Alg$ be an interactive estimation algorithm with $\alpha$-large self-evaluations and a decaying
estimation error with $C_{T,\delta}\ge\ln(2T)$,
and suppose that we run
Algorithm~\ref{alg:MAIN12} with
  $\Alg$ as the base algorithm,
  $T\ge64d^{2.5}(C_{T,\delta/2})/\epsilon^2$,
  $n_1=\lceil\log_2(4/\delta)\rceil$,
  and
  $n_2=\lceil 128\ln(8n_1/\delta)/\epsilon^2\rceil$.
Then, with probability at least $1-\delta$, the output $\zhat \in \Z$ satisfies
\[
  \score(\zhat\given z^*) \ge \alpha - \epsilon,
\]
and the overall number of issued queries is
${O}\bigParens{\smash{\frac{d^{2.5}(C_{T,\delta/2}) +\ln^2(1/\delta)}{\epsilon^2}}}$.

In the deterministic setting,
it suffices to run $\Alg$
with $T>2d^{1.5}\ln(4/\epsilon)$ and return $\hat{z} = z_{\hat{t}}$ where $\hat{t} = \argmax_{t\in\{1,\dotsc,T\}} r_t$ is the index of the largest observed reward.
Then, with probability~1, we obtain $\score(\zhat\given z^*) \ge \alpha - \epsilon$ and issue at most $O(d^{1.5}\ln(4/\epsilon))$ queries.
\end{theorem}

\textbf{Unknown $\alpha^*$ and optimism.\;\;} Algorithms \ref{alg:MAIN1} and \ref{alg:MAIN12} achieve performance guarantees with respect to a provided optimality level $\alpha\le\alpha^*$. When it is not easy to provide a non-trivial $\alpha$ (for example, when $\alpha^*$ is unknown and cannot be non-trivially bounded), Algorithm~\ref{alg:OPTIMISM} uses the optimistic least squares algorithmic template (see, e.g., \citep{russo2013eluder}) to ensure $\alpha^*$-large self-evaluations with high probability and to achieve a sublinear $\regret(T,\alpha^*)$. Algorithm~\ref{alg:OPTIMISM} takes as input a confidence radius parameter~$R$ of the same order as the decaying estimation error parameter $\smash{C_{T,\delta}}$ for Algorithm~\ref{alg:MAIN1}. We can then show that $z^* \in\Z_t$ with high probability for all $t\in\{1,\dotsc,T\}$. Therefore, $z_t$ must satisfy $\score(z_t\given z_t) \geq \score(z^*\given z^*)=\alpha^*$. In %
Appendix~\ref{section::least_squares} we show this modified version of the algorithm satisfies the decaying estimation error property. This technique allows us to achieve a sublinear $\regret(T,\alpha^*)$ without knowing $\alpha^*$ beforehand.
Similar to the case of fixed $\alpha$, it is possible to derive a version of Algorithm~\ref{alg:OPTIMISM} that leverages an online regression oracle. (See Appendix~\ref{appendix::regression_oracle} for details.)\looseness=-1

\begin{algorithm}[t]
    \caption{Optimistic Interactive Estimation via Least Squares}\label{alg:OPTIMISM}
    \begin{algorithmic}[1]
        \STATE \textbf{Input:} {set of alternatives} $\Z$, evaluation function $\rho$, number of steps $T$, confidence-set radius $R$.
        \FOR{$t=1,\dotsc,T$}
     \STATE Compute confidence set
     \begin{align*}
        \zhat_t&=\argmin_{z\in \Z}\smash[t]{\sum_{i=1}^{t-1} \BigParens{\score(z_i\given z) - \reward_i}^2}, \label{eq:zhat_optimism}
    \\
       \Z_t &= \biggBraces{z\in\Z:\:\smash[b]{\sum_{i=1}^{t-1}} \BigParens{\score(z_i\given z) - \score(z_i\given\hat{z}_t )}^2 \leq R}.
\end{align*}
        \STATE Submit the query
               $
               z_t = \smash{\argmax_{z\in\Z_t } \score(z\given z)}$.
        \STATE Observe reward $\reward_t$.
        \ENDFOR
    \end{algorithmic}
\end{algorithm}

\section{Statistical queries}\label{sec:sq}

In this section we consider the statistical query (SQ) model, as defined in Example \ref{ex:SQ}.
In particular, we study the connection between our generalized framework and SQ learning, showing specifically that the \newdim dimension can be used to recover generalization bounds based on a known combinatorial parameter that characterizes SQ learning, called the \emph{strong SQ dimension}.
There are several notions of such a dimension
\citep{feldman2012complete, szorenyi2009characterizing}.
Here we focus on the one due to \citet{szorenyi2009characterizing}:

\begin{definition}[Strong SQ dimension, \citep{szorenyi2009characterizing}]
\label{def:ssqd}
For a fixed distribution $D$ over $\X$, the \emph{strong SQ dimension} of a {hypothesis} class $\H \subseteq \{\pm1\}^\X$ with respect to some $\epsilon >0$, denoted $\dimSQ(\H,\epsilon)$, is the largest number $d$ for which there exist $h_1,\dotsc, h_d \in \H$
such that:
\begin{enumerate}[noitemsep, label={\upshape(\alph*)}]
    \item $| \langle h_i, h_j \rangle|\le 1 - \epsilon $ for all $1 \le i< j \le d$, and
    \item $|  \langle h_i, h_{j} \rangle-  \langle h_{i'}, h_{j'} \rangle| \le \frac{1}{d}$ for all $1 \le i< j \le d$, $1 \le i'< j' \le d$,
\end{enumerate}
where $ \langle h, h' \rangle := \E_{x \sim D} [h(x)h'(x)]$.
\end{definition}

The \newdim and strong SQ dimensions are closely related to one another in the sense of each providing a kind of polynomial bound on the other, as stated in the next proposition
(see Appendix~\ref{sec:proof:lem:ssqd_to_newdim} for the proof).

\begin{proposition}\label{lem:ssqd_to_newdim}
  Let $D$ be a fixed distribution over $\X$,
  and let $\H \subseteq \{\pm1\}^\X$ be a hypotheses class.
  For $\epsilon >0$,
  let $\dsq(\epsilon)=\dimSQ(\H,\epsilon)$,
  and let $\drho(\epsilon) =\d_{\scoreSQ}(\H, 1, \epsilon)$.

  If $\drho(\epsilon)\geq 2$ then
    \begin{equation} \label{eq:lem:ssqd_to_newdim:1}
      \min\BigBraces{ \dsq(\epsilon),\,
                       \floor{4\epsilon^2\,(\dsq(\epsilon))^2} }
      \leq
      \drho(\epsilon)
      \leq
      \max\BigBraces{ \dsq(\epsilon/4),\,
                       4\epsilon^2\, (\dsq(\epsilon/4)+1)^2
                     }.
    \end{equation}
    Similarly, if $\drho(4\epsilon)\geq 2$ then
    \begin{equation} \label{eq:lem:ssqd_to_newdim:2}
      \min\bracesauto{ \drho(4\epsilon),\,
                       \floor{\frac{\sqrt{\drho(4\epsilon)}}
                                    {8\epsilon}
                             }
                     }
      \leq
      \dsq(\epsilon)
      \leq
      \max\bracesauto{ \drho(\epsilon),\,
                       \frac{\sqrt{\drho(\epsilon)+1}}
                                    {2\epsilon}
                     }.
    \end{equation}
\end{proposition}

We next give a lower bound based on the strong SQ dimension, which together with Proposition~\ref{lem:ssqd_to_newdim} will allow us to lower bound sample complexity of any interactive estimation algorithm in the SQ setting in terms of the \newdim dimension.

\begin{theorem}[SQ lower bound]\label{thm:lb}
Let $\epsilon > 0$, and let $\H \subseteq \{\pm 1\}^\X$ be a hypothesis class with strong SQ dimension $\dsq=\dimSQ(\H, 2\epsilon) \ge 11$.
Let $\Alg$ be any interactive estimation algorithm with the property that for any
target $h^* \in \H$, $\Alg$ outputs an $\epsilon$-approximation to $h^*$ with probability at least $2/3$ using at most $m$ queries. Then $m > \sqrt[3]{\dsq}/12$.
\end{theorem}

The proof relies on a reduction to a lower bound of \citet{szorenyi2009characterizing}. However,  the lower bound of \citet{szorenyi2009characterizing} holds within an SQ model that differs from ours, in that it allows adversarial query responses. Therefore, we first need to show how to obtain a learning algorithm $\Alg'$ that can be used with an adversarial oracle from an interactive estimation algorithm $\Alg$ that uses an unbiased stochastic query oracle (as we assume in this work). To do this, we apply the reduction technique developed by \citet{feldman2017statistical}.
(See Appendix~\ref{appendix:sq_lb} for the full proof and additional details.)

Combining Theorem~\ref{thm:lb} and Proposition \ref{lem:ssqd_to_newdim} yields a lower bound on the sample complexity of interactive estimation in the SQ setting, for a sufficiently small $\epsilon$,
in terms of the \newdim dimension:

\begin{corollary}\label{cor:sq_dissim_lb}
Let $\epsilon>0$, and let $\H \subseteq \{\pm 1\}^\X$ be a hypothesis class with strong SQ dimension $\dimSQ(\H, 2\epsilon) \ge 11$.
Let  $\drho(\epsilon) = \d_{\scoreSQ}(\H,1,\epsilon)$.
Assume
$\epsilon \le 1/\bigParens{2\sqrt{\drho(\epsilon)}}$.
Let $\Alg$ be any interactive estimation algorithm with the property that for any
target $h^* \in \H$, $\Alg$ outputs an $\epsilon$-approximation to $h^*$ with probability at least $2/3$ using at most $m$ queries.
Then $m > \sqrt[3]{\drho(\epsilon)}/12$.
\end{corollary}

\section{Bandits} \label{sec:bandits}

In this section we focus on the bandits setting described in Example~\ref{ex:bandits}.
We study the relationship between the \newdim dimension and the \emph{eluder dimension}~\citep{russo2013eluder}, a common combinatorial dimension for bounding regret of bandit algorithms. We show that eluder dimension can be used to upper bound the \newdim dimension, and we also highlight the cases when \newdim dimension leads to a tighter analysis.

Throughout this section we follow the setup introduced in Example~\ref{ex:bandits}. We consider an action set~$\A$, a class $\F$ of reward functions $f:\A\to[-1,1]$, and a target reward function $f^*\in\F$. We map this to our setting by considering the set of alternatives $\Z=\F\times\A$, evaluation function $\scoreBandits\bigParens{(f,a)\bigGiven(f', a')}={f'(a)}$ and the target $(f^*,a^*)$,
where $a^*=\argmax_{a\in\A} f^*(a)$.

\subsection{Comparison with eluder dimension}

We start by describing the relationship between our dimension and the eluder dimension.
Following \citet{russo2013eluder}, we define $\epsilon$-dependence and $\epsilon$-eluder dimension as follows:

\begin{definition}[$\epsilon$-dependence]
An action $a \in \A$ is \emph{$\epsilon$-dependent} on actions $\{ a_1, \dotsc, a_n\} \subseteq \A$ with respect to $\F$ if any pair of functions $f, f' \in \F$ satisfying $\sqrt{ \sum_{i=1}^n (f(a_i) - f'(a_i))^2 } \leq \epsilon $ also satisfies $\abs{f(a)-f'(a)}\leq\epsilon$. Furthermore, an action $a$ is \emph{$\epsilon$-independent} of $\{a_1,\dotsc, a_n\}$ with respect to $\F$ if it is not $\epsilon$-dependent on $\{a_1, \dotsc, a_n\}$.
\end{definition}

\begin{definition}[$\epsilon$-eluder {dimension}]
The \emph{$\epsilon$-eluder} dimension $\dimE(\F, \epsilon)$ is the length $d$ of the longest sequence of elements in $\A$ such that every element is $\epsilon$-independent of its predecessors. Moreover, the \emph{monotone eluder dimension} is defined as
$\dimEmon(\F, \epsilon)\coloneqq\max_{\epsilon' \ge \epsilon} \dimE(\F, \epsilon')$.
\end{definition}

The next theorem shows that the \newdim dimension is upper bounded by the eluder dimension (see Appendix \ref{appendix::upper_eluder_proof} for a proof):
\begin{theorem} \label{theorem::upper_eluder}
Let $\Z = \F \times \A$, $\score=\scoreBandits$, $\epsilon>0$, $\alpha\le\alpha^*$. Then %
$
 \dmon_{\score}( \Z, \alpha, 3\epsilon/2)
 \le
 9\,\dimEmon(\F,\epsilon)
$.
\end{theorem}
 Nevertheless, as the next example shows, the eluder dimension can be arbitrarily large, while the \newdim  dimension remains constant. In this example, the action set is a circle in $\reals^2$, that is,
$\A=\C\coloneqq\{\vb\in\reals^2:\:\norm{\vb}=1\}$. We fix two open semicircles $U_0,U_1\subseteq\C$ with positive $x$ and $y$ coordinates, respectively, and
for any $N \in \mathbb{N}$ and $\epsilon>0$, construct a function class $\F_{N,\epsilon}$ with all the functions $f: \A \to [-1,1]$ obtained by the following process.
First, pick one of the semicircles~$U_j$ and any $N$ points from $U_j$. On each of these points, $f$ can equal either $+\epsilon$ or $-\epsilon$. Everywhere else in~$U_j$, $f$ equals zero, and everywhere outside $U_j$, it equals the linear function $\langle\vb,\ab\rangle$ parameterized by some $\vb\in\C\setminus U_j$. %
Thus, the functions are constructed to be ``simple'' (namely, linear) near the optimal action $\vb$,
but complex far from it. The eluder dimension is large to capture overall complexity, whereas the dissimilarity dimension is small to capture the simplicity near the optimum.
(See Appendix \ref{appendix:separation_eluder_dissimilarity_pf} for
the formal construction of $\smash{\F_{N,\epsilon}}$ and
the proof of Proposition \ref{prop::separation_eluder_dissimilarity}.)
\begin{proposition}
\label{prop::separation_eluder_dissimilarity}
Let $\epsilon \in (0, 1/2)$, $N \in \mathbb{N}$ and consider the action set $\A=\C$. Then, there is a function class $\F_{N,\epsilon} \subseteq [-1, 1]^\A$, such that for $\Z_{N, \epsilon}\coloneqq \F_{N, \epsilon}\times\A$, $\score=\scoreBandits$, it holds that
$\d_{\score}(\Z_{N, \epsilon}, 1, \epsilon) \leq 16$, but the eluder dimension is lower bounded as
$\dimE(\F_{N,\epsilon},\epsilon)\geq N$.
\end{proposition}

Thus, our regret bound based on the dissimilarity dimension implies that (optimistic) least squares algorithms have a regret independent of $N$. The same analysis with the eluder dimension~\citep{russo2013eluder} yields a regret bound scaling polynomially with $N$.
This shows that in the cases when the function classes are simple near the optimum, but complex far from it, the dissimilarity dimension can better capture the statistical complexity of bandit optimization
than the eluder dimension.

\subsection{Dissimilarity dimension bounds}\label{subsec:bandit_examples}

We next derive \newdim dimension bounds for several standard bandit classes. Existing bounds on eluder dimension can be used to immediately bound the dissimilarity dimension, but in several cases we are able to obtain tighter bounds.

We first consider \emph{linear bandits}.
Let $\ball_n=\{\vb \in \reals^{n}:\: \norm{\vb}\le 1\}$ be the unit ball in~$\reals^n$.
Actions are chosen from a set $\A\subseteq\ball_n$;
the reward function class is $\F^{lb} = \{f_{\thetab} : \thetab \in \Theta\}$,
where $\Theta \subseteq \ball_n$ and $f_{\thetab}(\ab) = \langle \thetab, \ab\rangle$.
The corresponding set of alternatives is denoted $\Z^{lb} = \F^{lb} \times \A$. In this case we obtain the following bound (see Appendix~\ref{appendix:linear_dim_bound_pf} for a proof):
\begin{theorem}[Linear bandits]\label{theorem::linear_dim_bound}
Let $\Z^{lb}$ be as defined above, let $\score=\scoreBandits$, and let $\epsilon>0$, $\alpha\le\alpha^*$. Then $\d_{\score}(\Z^{lb}, \alpha,  \epsilon) \le 4n+3$.
Moreover, when $\alpha=1$, then $\d_{\score}(\Z^{lb}, \alpha,  \epsilon) \le 2n+1$.
\end{theorem}

The proof proceeds by deriving an upper bound as well as a lower bound on the rank of the matrix~$\Mb$ with entries $M_{ij}=\rho(z_i\given z_j)-c$ obtained from elements $z_1,\dotsc,z_d$ that satisfy the dimension condition for $d = \d_\score(\Z, \alpha,  \epsilon)$ with a scalar $c$. The upper bound on the rank is $n+1$, and the lower bound is $d/4$ (which can be tightened to $d/2$ when $\Mb$ is symmetric). The upper bound is obtained by basic linear algebra and the lower bound from a standard result on ranks of perturbed identity matrices \citep[Lemma 2.2]{alon2009perturbed}.
Combining these bounds then yields the claim of Theorem~\ref{theorem::linear_dim_bound}.
Similar to existing bounds on eluder dimension \citep[Proposition 6]{russo2013eluder}, our bound in Theorem \ref{theorem::linear_dim_bound} is linear in $n$. However, the eluder dimension bound has an additional dependence on $1/\epsilon$, while our bound does not.\looseness=-1

Next, we consider \textit{generalized linear model} (GLM) bandits. Similar to linear bandits, the action set is $\A\subseteq\ball_n$, but the function class includes a nonlinearity. Specifically, we are provided with a function $g:\reals\to\reals$ that is differentiable and strictly increasing, and consider the function class $\F^{glm} = \{f_{\thetab} : \thetab \in \Theta\}$ where $\Theta\subseteq\ball_n$ and $f_{\thetab}(\ab)=g(\langle \thetab,  \ab \rangle)$.
Furthermore, we assume that there are $\underline{h}, \overline{h} >0$ such that for all $\ab\in\A$,
$\thetab\in\Theta$, we have  $\underline{h} \le g'(\langle \thetab,  \ab \rangle)\le\overline{h}$. Define $r =\overline{h} / \underline{h}$. We again denote $\Z^{glm} = \F^{glm} \times \A$.
Using an existing bound on the eluder dimension for GLM bandits (\cite{russo2013eluder}, Proposition 7) and the fact that our dimension is bounded by the eluder dimension (Theorem~\ref{theorem::upper_eluder}), we obtain the following bound (see Appendix \ref{appendix:glm_dim_bound_pf} for a proof):
\begin{theorem}[GLM bandits]\label{theorem::glm_dim_bound}
Let $\Z^{glm}$ be as defined above, let $\score=\scoreBandits$, and let $\epsilon>0$, $\alpha\le\alpha^*$. Then $\d_{\score}(\Z^{glm}, \alpha,  \epsilon) \le O(n r^2 \log(\overline{h}/\epsilon))$.
\end{theorem}
By considering a different proof technique, along the lines of Theorem \ref{theorem::linear_dim_bound}, it might be possible to tighten this bound. We leave this extension for future work.

Next, we consider a bandit setting that is similar to GLMs, but in this case the non-linearity
is provided by the non-differentiable rectified linear unit (ReLU) activation function  $\relu(x) = \max\{x, 0\}$. We consider the action set
$\A=\ball_n$, and the set of reward functions $\F^{\relu}$ consisting of all functions of the form
$f_{\thetab,b}(\ab) = \relu(\langle \thetab, \ab \rangle - b)$
for some $\thetab \in \ball_n$ and $b \in [0,1)$. The subset of $\F^{\relu}$ with a fixed value of $b$ is denoted $\F^{\relu}_b$, and we consider the set of alternatives $\Z^{\relu}_b = \F^{\relu}_b \times \ball_n$.

Unlike the classes considered above, this setting can be shown to be challenging to learn in the general case. %
Indeed, it turns out that eluder dimension (as well as a related measure called star dimension) is growing at least exponentially with $n$~\citep{li2021eluder,dong2021provable}.  The same lower bound can be shown for the \newdim dimension by a similar proof technique. The following theorem also provides an exponential upper bound, showing that in certain regimes the exponential dependence is tight (see Appendix \ref{appendix:relu_dim_bound_pf} for a proof):

\begin{theorem}[ReLU bandits]\label{theorem::relu_dim_bound}
Let $\Z^{\relu}$ be as defined above, let $\score=\scoreBandits$, and let $\epsilon,b>0$ such that
${b\le 1-\epsilon}$. Then $\d_{\score}(\Z^{\relu}_b, 1-b,  \epsilon) = O \bigParens{\epsilon^{-n/2}}$, and  $\d_{\score}(\Z^{\relu}_{1-\epsilon}, \epsilon,  \epsilon) = \Omega\bigParens{\epsilon^{-n/2}}$.
\end{theorem}

We note that previous work (\cite{dong2021provable}, Theorem 5.1) has shown that for a function class of one-layer neural networks with ReLU activations, obtaining sublinear regret requires $T = \Omega(\epsilon^{-(n-2)})$.

\section{Conclusion} \label{sec:conclusion}

In this paper, we have introduced a new model for interactive estimation and proposed a new combinatorial dimension, called dissimilarity dimension, to study the hardness of learning in this model. In (stochastic, correlational) statistical query learning, our dimension is polynomially related
to the strong SQ dimension. In bandits, our dimension is upper bounded by the eluder dimension, and there are examples where the dissimilarity dimension leads to much tighter regret bounds.

While this work provides an initial investigation of the dissimilarity dimension, many open questions remain. For example, our regret bound for the general setting scales as $d^{1.25}$. Is it possible to tighten this to linear dependence, as is the case, for example, for eluder dimension? On the algorithmic side, we currently require solving a least squares problem of size $t$ in iteration $t$. Although we also introduce an algorithm that leverages an online regression oracle (see Appendix~\ref{appendix::regression_oracle}), the oracle-based approach still requires solving a least squares problem (on the data smoothed by the oracle). Is it possible to derive dissimilarity-dimension-based regret bounds directly for the predictions produced by the oracle? Ultimately, we hope investigations of relationships between dissimilarity dimension and related notions may help us understand the hardness of learning in interactive settings.

\bibliography{bib}
\bibliographystyle{plainnat}

\newpage
\appendix
\newpage
\section{Missing proofs of Section \ref{sec:upper}}\label{appendix:upper}

\subsection{Analysis of Least Squares Algorithms (Algorithms~\ref{alg:MAIN1} and~\ref{alg:OPTIMISM})}
\label{section::least_squares}

Our analysis relies on the following variant of
Freedman’s inequality \citep{freedman1975tail} (see \citet[Lemma 9]{agarwal2014taming} and \citet[Theorem 1]{beygelzimer2011contextual}).
\begin{lemma}[Simplified Freedman’s inequality]\label{lemma:super_simplified_freedman}
Let $R >0$ and let $X_1, \dotsc, X_n$
be a sequence of real-valued random variables, such that for all $i \in [n]$ it holds that $X_i \le R$ and $\E[ X_i \given X_1,\dotsc,X_{i-1}]=0$. For any $\delta \in (0,1)$, and $\eta \in (0,1/R)$, with probability at least $1-\delta$,
\begin{equation}
   \sum_{i=1}^{n} X_i \leq    \eta \sum_{i=1}^{n}  \E[ X_i^2\given X_1,\dotsc,X_{i-1}] + \frac{\ln(1/\delta)}{\eta}.
   \end{equation}
\end{lemma}

Next we define an $\epsilon$-cover of a set $\Z$, that will be used in the bound of Theorem~\ref{thm::new_least_squares_estimator}.

\begin{definition}[$\epsilon$-cover]
Let $\psi$ be the pseudometric over the set $\Z$ defined, for any $z_1,z_2 \in \Z$, as
\begin{equation} \label{eq::def_norm_cov}
    \psi(z_1, z_2) =
    \sup_{z \in \Z}\,
    \BigAbs{\score(z\given z_1) - \score(z\given z_2)}.
\end{equation}
We say a set $N \subseteq\Z$ is an \emph{$\epsilon$-cover} of $\Z$ with respect to $\psi$
if for every $z \in \Z$ there exists some
$z' \in  N$ such that $\psi(z, z') \le \epsilon$. We denote by $\mathcal{N}(\Z, \epsilon)$ the minimum cardinality of
 any $\epsilon$-cover of $\Z$.
\end{definition}

For example, in the case of linear bandits (see Section~\ref{subsec:bandit_examples}) when $\Z = \Z^{lb}$ and $\score = \scoreBandits$, it can be shown that $\mathcal{N}(\Z^{lb}, \epsilon)$ is upper bounded by the $\ell_2$-covering number of the $n$-dimensional unit ball. This is because for any $z, z_1, z_2 \in \Theta \times \mathcal{A}$,
\[
\bigAbs{\rho(z\given z_1 ) - \rho(z\given z_2)}
  =
    \bigAbs{
      \langle \thetab_1, \ab \rangle - \langle \thetab_2, \ab \rangle
    }
  \leq
    \norm{\thetab_1 - \thetab_2}\norm{\ab}
  \leq
    \norm{\thetab_1 - \thetab_2}.
\]
The bound on $\mathcal{N}(\Z^{lb}, \epsilon)$ now follows because
the $\ell_2$-covering number of the unit ball with radius $\epsilon$ is $O\bigParens{(3/\epsilon)^n}$ (see, for example, Lemma D.1 of \citet{du2021bilinear}).

We next show that Algorithm~\ref{alg:MAIN1} satisfies the decaying estimation error property with $C_{T,\delta}$ that scales logarithmically with the covering number with respect to $\psi$.

\begin{theorem}[LS guarantee]\label{thm::new_least_squares_estimator}
Consider the setting from Section \ref{sec:setting}, where the learner sequentially issues the queries $z_1, \dotsc, z_T $ and receives responses $r_1, \dotsc, r_T$.
Assume there is $\beta \ge 0$ such that $\bigAbs{r_t - \E[r_t\given z_t]}\le \beta$ for all $t$, and $\beta'\ge 2\beta$ such that for all $z,z'$,
$\bigAbs{\score(z\given z') - \score(z\given z^*)} \le \beta'$.
Let $\smash{\tilde{\Z}}$ be a set of alternatives such that $z^* \in\smash{\tilde{\Z}}$ and let $\hat{z}_{t}$ be defined as the least squares optimizer,
\begin{equation*}
    \hat{z}_t = \argmin_{z \in \tilde{\Z}} \sum_{i=1}^{t-1} \BigParens{\score(z_i\given z) - \reward_i}^2.
\end{equation*}
Then, for any sequence of queries $z_1,\dotsc,z_T\in\tilde{\Z}$ (possibly equal to $\hat{z}_1, \dotsc, \hat{z}_T$), we have with probability $1-\delta$, for all $t \in [T]$ simultaneously, %
\begin{align*}
&
 \sum_{i=1}^{t-1}
   \BigParens{
     \score(z_i\given \hat{z}_{t}) - \score(z_i\given z^*)
   }^2
 \le C_{T,\delta},\text{ and}
\\
&
  z^* \in \biggBraces{z\in\tilde{\Z}:\:\sum_{i=1}^{t-1} \BigParens{\score(z_i\given z) - \score(z_i\given\hat{z}_t )}^2 \leq C_{T,\delta}},
\end{align*}
where $C_{T,\delta} = 16\beta \beta'\ln\bigParens{
      2T\mathcal{N}(\tilde{\Z}, \beta'\!/T) \bigm/ \delta
    }$.
\end{theorem}

\begin{proof}%
For $i=1,\dotsc,T$, let $h_i=(z_1,r_1,\dotsc,z_{i-1},r_{i-1},z_i)$ denote the history of interaction up to the query $z_i$, but excluding the response $r_i$, and let
$\xi_i=\reward_i-\score( z_i \given z^*)$. In the interactive estimation setting, we then have $\E[\xi_i\given h_i] = 0$, and by the lemma assumption, $\E[\xi_i^2\given h_i] \le \beta^2$.

Since $\hat{z}_{t}$ is the minimizer of the least squares loss up to time $t$,
we have
\begin{equation*}
  \sum_{i=1}^{t-1} \BigParens{\score(z_i\given \hat{z}_{t}) - \reward_i}^2
  \leq
  \sum_{i=1}^{t-1} \BigParens{\score(z_i\given z^*) - \reward_i}^2,
\end{equation*}
which can be rewritten, substituting $\reward_i = \score(z_i\given z^*)  + \xi_i$, as
\begin{equation*}
  \sum_{i=1}^{t-1}
    \BigParens{\score(z_i\given \hat{z}_{t}) - \score(z_i\given z^*) - \xi_i
    }^2
  \leq
  \sum_{i=1}^{t-1} \xi_i^2.
\end{equation*}
Therefore, by re-arranging terms, we get
\begin{equation}
\label{eq::main_ls_ineq}
  \sum_{i=1}^{t-1}
    \BigParens{
      \score(z_i\given \hat{z}_{t} ) - \score(z_i\given z^*)
    }^2
  \leq
  2\sum_{i=1}^{t-1}
    \xi_i\BigParens{
      \score(z_i\given \hat{z}_t ) - \score(z_i\given z^*)
    }.
\end{equation}

Set $\epsilon_1=\beta'\!/T$, and let $N$ be a minimal $\epsilon_1$-cover of $\tilde{\Z}$ with respect to the pseudometric $\psi$ (see Eq.~\ref{eq::def_norm_cov}). Furthermore, let
$\hat{z}_t^\epsilon\in N$ be an element of this cover that is $\epsilon_1$-close to $\hat{z}_t$ (with respect to~$\psi$). Then,
\begin{align}
\notag
    \sum_{i=1}^{t-1}
      \BigParens{
        \score(z_i\given \hat{z}_t ) - \score(z_i\given z^*)
      }^2
  &\leq
    2\sum_{i=1}^{t-1}
      \xi_i\BigParens{
         \score(z_i\given \hat{z}_t ) - \score(z_i\given z^*)
      }
\\
\notag
  &=
    2\sum_{i=1}^{t-1}
      \xi_i\Bigl(
        \score(z_i\given \hat{z}_t ) - \score(z_i\given \hat{z}_t^\epsilon)
\\
\notag
  &\hphantom{{}=
      2\sum_{i=1}^{t-1}
        \xi_i\Bigl(
          \score(z_i\given \hat{z}_t )
    }
    + \score(z_i\given \hat{z}_t^\epsilon) - \score(z_i\given z^*)
    \Bigr)
\\
\label{eq:ls_b_abc}
  &\le
    2t\beta\epsilon_1
    +
    2\sum_{i=1}^{t-1}
      \xi_i\BigParens{
        \score(z_i\given \hat{z}_t^\epsilon) - \score(z_i\given z^*)
      },
\end{align}
where the first inequality follows from Eq.~\eqref{eq::main_ls_ineq}, and the
last inequality follows because $\abs{\xi_i} \le \beta$ and $\hat{z}_t^\epsilon$ is $\epsilon_1$-close to $\hat{z}_t$.

Now, for any $z \in N$ and $i \in [T]$, define
\begin{equation*}
    K^z_i =
    \xi_i \BigParens{
      \score(z_i\given z) - \score(z_i\given z^*)
    }.
\end{equation*}
Since $\E[\xi_i\given h_i] = 0$, we have, for any \emph{fixed} $z\in N$, $\E[K^z_i\given h_i]=0$. This means that for any fixed $z\in N$, $K_1^z,\dotsc,K_T^z$ is a martingale difference sequence.
By the lemma assumptions, $|K^z_i| \le \beta\beta'$. Also,
\begin{equation}
  \E\bigBracks{
    (K^z_i)^2
  \bigGiven
    h_i
  }
\le
  \beta^2
  \E\BigBracks{
    \bigParens{\score(z_i\given z) - \score(z_i\given z^*)
    }^2
  \BigGiven
    h_i
  }
=
  \beta^2
  \bigParens{\score(z_i\given z) - \score(z_i\given z^*)}^2.
\end{equation}
 Thus, by Freedman's inequality (Lemma~\ref{lemma:super_simplified_freedman}) with $\eta=1/(4\beta \beta')$ and $\delta'=\frac{\delta}{T\card{N}}$,
 we obtain that for any fixed $z\in N$ and $t \in [T]$, with probability at least $1-\delta'$,
\begin{align}
\notag
    \sum_{i=1}^{t-1}
    \xi_i\BigParens{
      \score(z_i\given z) - \score(z_i\given z^*)
    }
&\leq
    \frac{1}{4\beta \beta'}
    \sum_{i=1}^{t-1}
    \beta^2\BigParens{
      \score(z_i\given z) - \score(z_i\given z^*)
    }
    + 4\beta \beta'\ln\biggParens{\frac{T\card{N}}{\delta}}
\\
\label{eq::free_ineq}
&=
    \frac{\beta}{4\beta'}
    \sum_{i=1}^{t-1}
    \BigParens{
      \score(z_i\given z) - \score(z_i\given z^*)
    }^2\!
    + 4\beta \beta'\ln\biggParens{\frac{T\card{N}}{\delta}}.
\end{align}
Taking a union bound over all $z\in N$ and $t\in [T]$, we obtain that
Eq.~\eqref{eq::free_ineq} holds with probability at least $1-\delta$ simultaneously for all $z\in N$ and $t \in [T]$. Henceforth, we assume that we are in the event
when Eq.~\eqref{eq::free_ineq} holds for all $z\in N$ and $t \in [T]$.

Applying the bound of Eq.~\eqref{eq::free_ineq}
with $z=\hat{z}_t^\epsilon$
to the sum on the right-hand side of Eq.~\eqref{eq:ls_b_abc} then yields
\begin{equation}
\label{eq:inter}
\begin{multlined}
    \sum_{i=1}^{t-1}
      \BigParens{
        \score(z_i\given \hat{z}_t ) - \score(z_i\given z^*)
      }^2
   \le
    2t\beta\epsilon_1
    +
    \frac{\beta}{2\beta'}
    \sum_{i=1}^{t-1}
    \BigParens{
      \score(z_i\given \hat{z}_t^\epsilon) - \score(z_i\given z^*)
    }^2
\\{}
    +
    \smash[t]{
        8\beta \beta'\ln\biggParens{\frac{T\card{N}}{\delta}}
    }.
\end{multlined}
\end{equation}
Using the inequality $(a+b)^2\le 2a^2+2b^2$, which holds for any $a,b\in\reals$, and the fact that $\hat{z}_t^\epsilon$ and $\hat{z}_t$ are $\epsilon_1$-close, we obtain, for every $i=1,\dotsc,t-1$,
\begin{align*}
  \BigParens{
      \score(z_i\given \hat{z}_t^\epsilon) - \score(z_i\given z^*)
    }^2
&=
  \BigParens{
      \bigBracks{
         \score(z_i\given \hat{z}_t^\epsilon)
         - \score(z_i\given \hat{z}_t)
      }
      +
      \bigBracks{
         \score(z_i\given \hat{z}_t)
         - \score(z_i\given z^*)
      }
    }^2
\\
&\le
  2\epsilon_1^2
  +
  2\bigParens{
         \score(z_i\given z^*)
         - \score(z_i\given \hat{z}_t)
  }^2.
\end{align*}
Plugging this into the right-hand side of Eq.~\eqref{eq:inter} yields
 \begin{align*}
    \sum_{i=1}^{t-1}
      \BigParens{
        \score(z_i\given \hat{z}_t ) - \score(z_i\given z^*)
      }^2
 &\le
    2t\beta\epsilon_1
    +
    \frac{\beta}{\beta'}t\epsilon_1^2
    +
    \frac{\beta}{\beta'}
    \sum_{i=1}^{t-1}
    \BigParens{
      \score(z_i\given \hat{z}_t) - \score(z_i\given z^*)
    }^2
\\&\qquad\qquad\qquad\qquad\qquad{}
    +
    \smash[t]{
        8\beta \beta'\ln\biggParens{\frac{T\card{N}}{\delta}}
    }
\\
 &\le
    2t\beta\epsilon_1
    +
    \frac{t\beta\epsilon_1^2}{\beta'}
    +
    \frac12
    \sum_{i=1}^{t-1}
    \BigParens{
      \score(z_i\given \hat{z}_t) - \score(z_i\given z^*)
    }^2
\\&\qquad\qquad\qquad\qquad\qquad{}
    +
    \smash[t]{
        8\beta \beta'\ln\biggParens{\frac{T\card{N}}{\delta}}
    },
\end{align*}
where the last inequality follows by the assumption that $\beta' \ge 2\beta$.
Then, by re-arranging terms and multiplying by $2$, we get
\[
 \sum_{i=1}^{t-1}
   \BigParens{
     \score(z_i\given \hat{z}_t  ) - \score(z_i\given z^*)
   }^2
 \le
    4t\beta\epsilon_1
    +
    \frac{2t\beta\epsilon_1^2}{\beta'}
    +
    16\beta \beta'\ln\biggParens{
      \frac{T\card{N}}{\delta}
    }.
\]
Recall that we set $\epsilon_1 = \beta'\!/T$, and $N$ is a minimal $\epsilon_1$-cover of $\tilde{\Z}$, so $\card{N}=\mathcal{N}(\tilde{\Z}, \beta'\!/T)$. Plugging these values in the previous equation, we thus obtain
that with probability at least $1-\delta$, for all $t \in [T]$,
\begin{align}
\notag
 \sum_{i=1}^{t-1}
 \BigParens{
   \score(z_i\given \hat{z}_t ) - \score(z_i\given z^*)
 }^2
&
 \le
    4\beta\beta'
    +
    \frac{2\beta\beta'}{T}
    +
    16\beta \beta'\ln\biggParens{
      \frac{T\mathcal{N}(\tilde{\Z}, \beta'\!/T)}{\delta}
    }
\\
\label{equation::ls_guarantee_decaying_error}
&\le
    16\beta \beta'\ln\biggParens{
      \frac{2T\mathcal{N}(\tilde{\Z}, \beta'\!/T)}{\delta}
    },
\end{align}
where the last inequality follows because $4+2/T\le 16\ln 2$ for $T\ge 1$.
Finally, when Eq.~\eqref{equation::ls_guarantee_decaying_error} holds, we also have
\[
  z^* \in
  \biggBraces{
    z\in\tilde{\Z}:\:
    \sum_{i=1}^{t-1} \BigParens{\score(z_i\given z) - \score(z_i\given\hat{z}_t )}^2
    \le
    16\beta \beta'\ln\biggParens{
      \frac{2T\mathcal{N}(\tilde{\Z}, \beta'\!/T)}{\delta}
    }
  }.
\qedhere
\]
\end{proof}

Considering Algorithm~\ref{alg:MAIN1} and using Theorem~\ref{thm::new_least_squares_estimator} with $\tilde{\Z}=\Z_\alpha$, $z_t = \hat{z}_t$, $\beta=2$ and $\beta'=4$ then immediately yields the following corollary
($\alpha$-large self-evaluations follow because $\hat{z}_t\in\Z_\alpha$):

\begin{corollary}
\label{cor:alg:MAIN1}
Consider the setting from Section \ref{sec:setting} with a set of alternatives $\Z$ and an evaluation function $\score$.
Let $\alpha$ be an optimality level such that
$\mathcal{N}({\Z}_\alpha, 4/T)=e^{o(T)}$. Then Algorithm~\ref{alg:MAIN1} has $\alpha$-large self-evaluations and satisfies
the decaying error property with
$C_{T,\delta} = 128\ln\bigParens{
      2T\mathcal{N}(\Z_\alpha, 4/T) \big/ \delta
    }$.
\end{corollary}

Similarly, Theorem~\ref{thm::new_least_squares_estimator} also implies that Algorithm~\ref{alg:OPTIMISM} satisfies the decaying error property as well as $\alpha^*$-large self-evaluations, although $\alpha^*$ is not known:

\begin{corollary}
\label{cor:alg:OPTIMISM}
Consider the setting from Section \ref{sec:setting} with a set of alternatives $\Z$ and an evaluation function $\score$,
and assume that $\mathcal{N}(\Z, 4/T)=e^{o(T)}$. Then Algorithm~\ref{alg:OPTIMISM}
with $R=128\ln\bigParens{
      2T\mathcal{N}(\Z, 4/T) \big/ \delta
     }$
has $\alpha^*$-large self-evaluations and satisfies
the decaying error property with
$C_{T,\delta} = 4R=512\ln\bigParens{
      2T\mathcal{N}(\Z, 4/T) \big/ \delta
    }$.
\end{corollary}
\begin{proof}
We apply Theorem~\ref{thm::new_least_squares_estimator} with $\tilde{\Z} = \Z$, $\beta=2$ and $\beta'=4$. Our choice of $R$ in Algorithm~\ref{alg:OPTIMISM} coincides with the value of $C_{T,\delta}$ appearing in Theorem~\ref{thm::new_least_squares_estimator}, and therefore the theorem implies that $z^*\in\Z_t$ with probability at least $1-\delta$ for all $t\in[T]$. In that case the queries $z_t$ issued by the algorithm satisfy
$
 \score(z_t\given z_t)\ge\score(z^*\given z^*)=\alpha^*
$
and thus the algorithm has $\alpha^*$-large self-evaluations.

For the second part, the triangle inequality implies that
with probability at least $1-\delta$ for all $t\in[T]$,
\begin{align}
\notag
  \sqrt{
    \sum_{i=1}^{t-1} \BigParens{\score(z_i\given z) - \score(z_i\given z_t )}^2
  }
&\leq
  \sqrt{
    \sum_{i=1}^{t-1} \BigParens{\score(z_i\given z) - \score(z_i\given\hat{z}_t )}^2
  }
  +
  \sqrt{
    \sum_{i=1}^{t-1} \BigParens{\score(z_i\given \hat{z}_t) - \score(z_i\given z_t )}^2
  }
\\
\notag
&\le
  \sqrt{R}+\sqrt{R},
\end{align}
where the bound on the first term on the right-hand side follows by Theorem~\ref{thm::new_least_squares_estimator} and the bound on the second term by the fact that $z_t\in\Z_t$.
\end{proof}

\subsection{Proof of Theorem \ref{thm:C_T_delta_finite_Z}}
The theorem follows immediately from Corollary~\ref{cor:alg:MAIN1},
because
$\mathcal{N}(\Z_\alpha, \epsilon) \le |\Z_\alpha|\le |\Z|<\infty$ for any $\alpha$ and $\epsilon$.

\subsection{Online Regression Oracles}\label{appendix::regression_oracle}

 We assume access to an \emph{online regression oracle} $\AlgOLS$, which solves a regression problem over a function class $\Phi=\set{\phi_z: z\in\Z}$ indexed by $z\in\Z$, where $\phi_z:\Z\to\reals$ is defined as $\phi_z(z')=\score(z'\given z)$ for all $z'\in\Z$; that is, functions $\phi_z$ evaluate $\score$ in its first argument.
 
 The oracle operates in the following protocol: In each time step, the oracle receives an observation $z_t$, produces a prediction $\widehat{\score}_{t} \in \mathbb{R}$, and finally receives a response $r_t$ and incurs square loss $(\widehat{\score}_{t}-r_t)^2$.
 We assume that for any $T$ and sequence of observations and responses (even if generated adaptively), the oracle satisfies the following regret bound:
 \begin{equation}
 \label{equation::regression_oracle_regret_guarantees}
  \sum_{t=1}^T (\widehat{\score}_{t}-  \reward_{t})^2 - \inf_{z \in \Z} \sum_{t=1}^T (\score( z_{t} \given z) -\reward_{t}  )^2 \leq \RegretOLS{T},
 \end{equation}
 where $\RegretOLS{\cdot}$ is a non-decreasing sublinear function (that typically also depends on various properties of $\rho$, $\Z$, and the range of responses $r_t$).
 For many function classes,
 there are well-known constructions of online regression oracles that satisfy Eq.~\eqref{equation::regression_oracle_regret_guarantees}~\citep{cover1991universal,vovk1998game,kalai2002efficient}.
 For example, if $\Phi$ is finite, there are oracles with $\RegretOLS{T} = O(\ln\card{\Phi})$ and for parametric classes, such as linear functions, there are oracles with $\RegretOLS{T} = O(d\log(T/d))$. More examples can be found in Section~2.2 of~\citet{foster2020beyond}.

\begin{algorithm}[H]
    \caption{Interactive Estimation via Least Squares}\label{alg:MAINonlineregression}
    \begin{algorithmic}[1]
        \STATE \textbf{Input:} online regression oracle $\AlgOLS$, optimality level $\alpha$.
        \STATE Initialize $z_1$ to an arbitrary element of $\Z_\alpha$.
        \FOR{$t=1$  \ldots  $T$}
        \STATE Use $\AlgOLS$ to predict $\widehat{\score}_{t}$ given the observation $z_t$.
        \STATE Observe reward $\reward_{t+1}$ and pass it to $\AlgOLS$.
        \STATE  Set
        $\displaystyle
        z_{t+1}=  \argmin_{z \in \Z_\alpha}\sum_{i=1}^{t} \Big(\score(z_i\given z) - \widehat{\score}_i\Big)^2
        $.
        \label{step:regoracle}
        \ENDFOR
    \end{algorithmic}
\end{algorithm}

We now analyze Algorithm~\ref{alg:MAINonlineregression} under the assumption of access to an online regression oracle. This algorithm takes as input an online regression oracle $\AlgOLS$. Algorithm~\ref{alg:MAINonlineregression} can also be modified using the optimistic least squares template of~\citet{russo2013eluder} to handle the case when $\alpha^*$ is unknown. This is done by replacing step~\ref{step:regoracle} of Algorithm~\ref{alg:MAINonlineregression} with the following two steps:
\begin{align*}
   \Z_{t+1}&=\BigBraces{z\in\Z:\:\smash[t]{\sum_{i=1}^{t}} (  \score(z_i \given z) - \widehat{\score}_i)^2 \leq R},
\\
   z_{t+1} &= \argmax_{z\in\Z_{t+1}} \score(z\given z),
\end{align*}
where $R =  8 \RegretOLS{T} + 64\beta \max\set{\beta, \beta'}\ln\bigParens{\frac{T}{\delta}}$ and $\beta, \beta'$ are defined as in Lemma~\ref{lemma::online_to_confidence_bound}. The results of Lemma~\ref{lemma::online_to_confidence_bound} justify the validity of these choices (using a similar reasoning as in Corollary~\ref{cor:alg:OPTIMISM}) and
imply that Algorithm~\ref{alg:MAINonlineregression} satisfies the decaying estimation error property (Definition~\ref{assume:est:error}), provided that the regression oracle $\AlgOLS$ satisfies the regret bound of Eq.~\eqref{equation::regression_oracle_regret_guarantees}.

\begin{lemma}\label{lemma::online_to_confidence_bound}
Consider the setting defined in Section \ref{sec:setting} with $\alpha\le\alpha^*$, and assume there are $\beta, \beta' \ge 0$ such that $|r_t - \E[r_t\given z_t]|\le \beta$ for all $t$ and there is $\beta'\ge 2\beta$ s.t. for all $z \in \Z$ and $\widehat{\rho} \in \Gamma_z$,  $ \big| \widehat{\rho} - \score(z\given z^*) \big| \le \beta'$ where $\Gamma_z \subset \mathbb{R}$ is the space of plausible responses of $\AlgOLS$ for input $z$.
The sequence of queries $z_1, \dotsc, z_T$ as defined in Algorithm~\ref{alg:MAINonlineregression} satisfies with probability at least $1-\delta$ for all $t \in [T]$ simultaneously,
\begin{align*}
\sum_{i=1}^{t} \left( \score(z_i \given z_{t+1}) - \score(z_i \given z^*) \right)^2 &\leq C_{T,\delta} \quad \text{ and }
\quad z^* \in
\BigBraces{ z\in\Z_\alpha:\: \sum_{i=1}^{t} (  \score(z_i \given z) - \widehat{\score}_i)^2 \leq C_{T,\delta}}
\end{align*}
where $C_{T,\delta} =  8\RegretOLS{T} + 64\beta \max\set{\beta, \beta'}\ln\bigParens{\frac{T}{\delta}}$.
\end{lemma}

\begin{proof}
Let $\xi_i = \reward_i - \score( z_i \given z^*) $. Recall that $\E[r_i \given z_i] = \score(z_i \given z^*)$ and therefore, by assumption, $|\xi_i| \leq \beta$ for all $i$. By definition, the online regression oracle satisfies
\begin{align}
\sum_{i=1}^{t} (\widehat{\score}_i -  \reward_i)^2 &\leq  \inf_{z \in \Z} \sum_{i=1}^{t} (\score( z_i \given z) -\reward_i  )^2 + \RegretOLS{t}\notag \\
&\stackrel{(i)}{\leq} \sum_{i=1}^{t} (\score( z_i \given z^*) -\reward_i  )^2 + \RegretOLS{t} \notag\\
&= \sum_{i=1}^{t} \xi_i^2 + \RegretOLS{t}. \label{equation::beginning_reg_oracle_eq}
\end{align}
Inequality $(i)$ holds because $z^* \in \Z$. Expanding the LHS,
\begin{equation*}
\sum_{i=1}^{t} (\widehat{\score}_i -  \reward_i)^2 = \sum_{i=1}^{t}
\BigBracks{(\widehat{\score}_i -  \score(z_i \given z^*))^2 - 2\xi_i ( \widehat{\score}_i -  \score(z_i \given z^*)  ) + \xi_i^2}.
\end{equation*}

Plugging this back into Eq.~\eqref{equation::beginning_reg_oracle_eq} and rearranging,
we obtain
\begin{equation}\label{eq::helper_eq_regoracle}
\sum_{i=1}^{t} (\widehat{\score}_i -  \score(z_i \given z^*))^2\leq \sum_{i=1}^{t} 2\xi_i ( \widehat{\score}_i -  \score(z_i \given z^*)  )  + \RegretOLS{t}.
\end{equation}
For any $i \in [T]$ we define:
\begin{equation*}
 K_i =  \xi_i ( \widehat{\score}_i -  \score(z_i \given z^*)  )
\end{equation*}
 Observe that %
\begin{equation*}
\mathbb{E}\left[ K_i   \given \{ z_\ell, \widehat{\score}_\ell \}_{\ell=1}^{i} \right] = 0.
\end{equation*}

Thus $K_1, \dotsc, K_T$ is a martingale difference sequence.
Notice that $| K_i | \leq \beta \beta'$ and that
\begin{equation*}
\mathbb{E}\bigBracks{K_i^2\bigGiven\{ z_\ell, \widehat{\score}_\ell \}_{\ell=1}^{i}} \le \beta^2 \mathbb{E}\Big[\bigl( \widehat{\score}_i - \score(z_i\given z^*) \bigr)^2\bigGiven\{ z_\ell, \widehat{\score}_\ell \}_{\ell=1}^{i}\Big] =  \beta^2 \bigl( \widehat{\score}_i  - \score(z_i\given z^*) \bigr)^2.
\end{equation*}
 Then, by plugging this into Freedman's inequality (Lemma~\ref{lemma:super_simplified_freedman}) with $\eta:=\frac{1}{4}\min(1/\beta^2, 1/\beta\beta')$ and $\delta':=\delta/T$, we get that for any fixed $t \in [T]$, with probability at least $1-\delta$,
\begin{equation}\label{eq::free_ineq_regoracle}
    \sum_{i=1}^{t} \xi_i \bigl( \widehat{\score}_i - \score(z_i\given z^*) \bigr) \leq      \frac{1}{4}\sum_{i=1}^{t} \bigl( \widehat{\score}_i  - \score(z_i\given z^*) \bigr)^2 + 4 \beta \max\set{\beta, \beta'} \ln\Parens{\frac{T}{\delta}}.
\end{equation}

Plugging Eq.~\eqref{eq::free_ineq_regoracle} back into Eq.~\eqref{eq::helper_eq_regoracle} and rearranging terms yields
\begin{equation}\label{equation::main_helper_inequalityregoracle}
\sum_{i=1}^{t} (\widehat{\score}_i -  \score(z_i \given z^*))^2 \leq 2 \RegretOLS{t}
+ 16\beta\max\set{\beta, \beta'} \ln\Parens{\frac{T}{\delta}}
\end{equation}
with probability at least $1-\delta$ for all $t \in [T]$. Thus we conclude that with probability at least $1-\delta$, for all $t \in [T]$,
\[
z^* \in
  \BigBraces{ z\in\Z_\alpha:\: \sum_{i=1}^{t} (  \score(z_i \given z) - \widehat{\score}_i)^2 \leq C_{T,\delta}
  }.
\]
By the triangle inequality,
\begin{equation*}
\sqrt{ \sum_{i=1}^{t} ( \score(z_i \given z^* ) - \score(z_i \given z_{t+1}))^2  } \leq \sqrt{ \sum_{i=1}^{t} ( \score(z_i \given z^* ) - \widehat{\score}_i)^2  }  + \sqrt{ \sum_{i=1}^{t} (  \widehat{\score}_i - \score(z_i \given z_{t+1}))^2  }.
\end{equation*}
Since by definition $z_{t+1} = \argmin_{z\in \Z_\alpha} \sum_{i=1}^{t} \Big(\score(z_i\given z) - \widehat{\score}_i\Big)^2$, we have
\[
  \sum_{i=1}^{t} \Big(\score(z_i\given z_{t+1}) - \widehat{\score}_i\Big)^2 \leq \sum_{i=1}^{t} \Big(\score(z_i\given z^*) - \widehat{\score}_i\Big)^2.
\]
Substituting back into the triangle inequality above,
\begin{equation*}
\sqrt{ \sum_{i=1}^{t} ( \score(z_i \given z^* ) - \score(z_i \given z_{t+1}))^2  } \leq 2\sqrt{ \sum_{i=1}^{t} ( \score(z_i \given z^* ) - \widehat{\score}_i)^2  },
\end{equation*}
implying
\begin{equation*}
 \sum_{i=1}^{t} ( \score(z_i \given z^* ) - \score(z_i \given z_{t+1}))^2   \leq 4 \sum_{i=1}^{t} ( \score(z_i \given z^* ) - \widehat{\score}_i)^2.
\end{equation*}
Plugging Eq.~\eqref{equation::main_helper_inequalityregoracle} on the right-hand side yields
\begin{equation*}
 \sum_{i=1}^{t} ( \score(z_i \given z^* ) - \score(z_i \given z_{t+1}))^2 \leq 8 \RegretOLS{t} + 64 \beta \max\set{\beta, \beta'} \ln\Parens{\frac{T}{\delta}}.
\end{equation*}
The result follows by using the monotonicity of $\RegretOLS{t}$.
\end{proof}

\subsection{Proof of Lemma \ref{lem:few_bad_queries}} \label{appendix:few_bad_queries_pf}

The proof uses Tur\'an's Theorem~\citep{turaan1941extremal}, a standard result from extremal graph theory that bounds the number of edges of a graph that does not contain a clique of a given size:
\begin{theorem}[Tur\'an's Theorem]
\label{theorem::turan}
Let $G=(V,E)$ be an undirected graph without self-loops and whose largest clique is of size at most $d$. Then
\[
  \Card{E}\le\Parens{1-\frac1d}\frac{\card{V}^2}{2}.
\]
\end{theorem}

We now turn to the proof of Lemma~\ref{lem:few_bad_queries}. First note that if $\epsilon\ge 1+\alpha$ then no query is $\epsilon$-bad, because $\alpha-\epsilon\le -1\le\score(z\given z^*)$ for every $z\in\Z$,
and therefore the lemma holds. In the remainder of the proof, we assume that $0<\epsilon<1+\alpha$.

Consider the queries $z_1,\dotsc,z_T$ and their corresponding values relative to $z^*$, denoted as $v_t=\score(z_t\given z^*)$ for $t\in[T]$. A query $z_t$ is $\epsilon$-bad if its corresponding value $v_t$ is in the interval $I=[-1,\alpha-\epsilon)$. The proof proceeds by partitioning the interval $I$ into subintervals and separately bounding the number of values $v_t$ in each subinterval.

To define these subintervals, let $q=1+\smash{\frac{1}{\sqrt{d}}}$, and consider the sequence of suboptimality gaps $\epsilon_i=q^{i-1}\epsilon$ for $i=1,\dotsc,n+1$, where
\[
  n=\Ceil{\log_q\Parens{\frac{1+\alpha}{\epsilon}}}.
\]
The gaps $\epsilon_i$ form an increasing sequence $\epsilon, q\epsilon, q^2\epsilon,\dotsc$ such that the last element satisfies
\begin{align*}
  \epsilon_{n+1}&=q^n\epsilon\ge\Parens{\frac{1+\alpha}{\epsilon}}\epsilon=1+\alpha.
\end{align*}
Using these gaps we define intervals $I_i=[\alpha-\epsilon_{i+1},\alpha-\epsilon_i)$ for $i=1,\dots,n$.
Since $\epsilon_{n+1}\ge 1+\alpha$, the union $I_1\cup\dots\cup I_n=[\alpha-\epsilon_{n+1},\alpha-\epsilon)$ covers the interval $I$. We bound the number of values $v_t$ in each interval $I_i$.

Let $S_i$ be the set of query indices with values in $I_i$, that is $S_i=\bigSet{t\in[T]:\:\score(z_t\given z^*)\in I_i}$, let $m_i=\card{S_i}$, and assume that $m_i\ge 2$ (the case $m_i\le 1$ will be dealt with later).
Furthermore, let $c_i=\alpha-(\epsilon_{i+1}+\epsilon_i)/2$ be the midpoint of the interval $I_i$. Since the width of the interval $I_i$ is $\epsilon_{i+1}-\epsilon_i=(q-1)\epsilon_i=\epsilon_i/\sqrt{d}$, we obtain
\begin{equation}
\label{eq:ci:bound}
  \bigAbs{\score(z_t\given z^*)-c_i}\le\frac{\epsilon_i}{2\sqrt{d}}
\end{equation}
for all $t\in S_i$.

Let $d_i=\d_\score(\Z, \alpha, \epsilon_i)$ be the (non-monotonic) \newdim dimension with respect to the suboptimality gap $\epsilon_i$. Since $\alpha\le\alpha^*$ and $\epsilon_i\ge\epsilon$, we have $1\le d_i\le d$. We construct an upper bound on~$m_i$, exploiting the fact that $d_i$ is the dissimilarity dimension with respect to~$\epsilon_i$.

In the rest of the proof we refer to a pair of queries with indices $s,t\in S_i$ such that $s<t$ as \emph{dissimilar} if
\[
  \bigAbs{\score(z_{s}\given z_t)-c_i}\le\smash[t]{\frac{\epsilon_i}{\sqrt{d_i}}}.
\]
This is exactly the property appearing in the definition of the dissimilarity dimension with respect to~$\epsilon_i$, and so there cannot be more than $d_i$ queries such that every pair is dissimilar (note that all the queries $z_t$ satisfy $\score(z_t\given z_t)\ge\alpha$ thanks to $\alpha$-large self-evaluations).

The derivation of the bound on $m_i$ proceeds in several steps. First, we identify pairs of dissimilar queries and construct a graph where each edge corresponds to a dissimilar pair.
Second, we use Tur\'an's Theorem (Theorem~\ref{theorem::turan}) to upper bound the number of such pairs, using the fact that the graph cannot contain a clique of size greater than $d_i$. Finally, using the bound on the number of dissimilar pairs, we bound $m_i$.

To start, let $t_1<t_2<\dotsc<t_{m_i}$ be the query indices included in $S_i$, and let $k\in\set{2,\dotsc,m_i}$. Consider a uniform distribution over $\ell\in[k-1]$. Then by Markov's inequality and the decaying estimation error property, we obtain
\begin{align*}
  \frac{1}{k-1}\sum_{\ell=1}^{k-1}
  \one\BiggBracks{
    \BigParens{\score(z_{t_\ell}\given z_{t_k}) - \score(z_{t_\ell}\given z^*)}^2
    \!
    \ge
    \frac{\epsilon_i^2}{4d}
  }
&\le
 \BiggBracks{
 \frac{1}{k-1}\sum_{\ell=1}^{k-1}
    \BigParens{\score(z_{t_\ell}\given z_{t_k}) - \score(z_{t_\ell}\given z^*)}^2
 }
 \!
 \cdot
 \frac{4d}{\epsilon_i^2}
\\
&\le
 \BiggBracks{
 \frac{1}{k-1}\sum_{s=1}^{t_k-1}
    \BigParens{\score(z_{s}\given z_{t_k}) - \score(z_{s}\given z^*)}^2
 }
 \!
 \cdot
 \frac{4d}{\epsilon_i^2}
\\
&\le
 \frac{C_{T,\delta}}{k-1}
\cdot
 \frac{4d}{\epsilon_i^2}.
\end{align*}
Multiplying by $k-1$, we therefore obtain
\[
  \Card{\Set{
    s\in S_i:\:
    s<t_k
    \text{ and }
    \BigAbs{\score(z_{s}\given z_{t_k}) - \score(z_{s}\given z^*)}
    \ge
    \frac{\epsilon_i}{2\sqrt{d}}
  }}
  \le
  \frac{4dC_{T,\delta}}{\epsilon_i^2},
\]
and summing across all $k\in\set{2,\dotsc,m_i}$ then yields
\begin{equation}
\label{eq:complement-E}
  \Card{\Set{
    s,t\in S_i:\:
    s<t
    \text{ and }
    \BigAbs{\score(z_{s}\given z_{t}) - \score(z_{s}\given z^*)}
    \ge
    \frac{\epsilon_i}{2\sqrt{d}}
  }}
  \le
  m_i\frac{4dC_{T,\delta}}{\epsilon_i^2}.
\end{equation}

We next construct an undirected graph without self-loops, $G_i=(V_i,E_i)$. The vertex set of the graph is $V_i=S_i$. The edge set is defined to be
\[
  E_i=\Set{
    \set{t,s}\subseteq S_i:\:
    s<t
    \text{ and }
    \BigAbs{\score(z_{s}\given z_{t}) - \score(z_{s}\given z^*)}
    <
    \frac{\epsilon_i}{2\sqrt{d}}
  }.
\]
By comparing with Eq.~\eqref{eq:complement-E}, we obtain
\begin{align}
\label{eq:bound-E}
  \card{E_i}
&\ge
  \frac{m_i(m_i-1)}{2}
  -
  m_i\frac{4dC_{T,\delta}}{\epsilon_i^2}.
\end{align}
Note that any pair of vertices $s<t$ connected by an edge corresponds to a dissimilar pair of queries:
\begin{align*}
  \bigAbs{\score(z_{s}\given z_{t}) - c_i}
&\le
  \bigAbs{\score(z_{s}\given z_{t}) - \score(z_{s}\given z^*)}
  +
  \bigAbs{\score(z_{s}\given z^*)-c_i}
\\
&<
  \frac{\epsilon_i}{2\sqrt{d}}
  +
  \frac{\epsilon_i}{2\sqrt{d}}
\\
&\le
  \frac{\epsilon_i}{\sqrt{d_i}},
\end{align*}
where the first inequality is the triangular inequality, the second inequality follows by
combining the definition of $E_i$ and Eq.~\eqref{eq:ci:bound}, and the final one is from the fact that
$d_i\le d$. From the definition of the dissimilarity coefficient, the largest clique in $G_i$ is of size at most $d_i$. Using Tur\'an's Theorem,
we thus must have
\[
  \card{E_i}\le\Parens{1-\frac{1}{d_i}}\cdot\frac{m_i^2}{2}
            \le\Parens{1-\frac{1}{d}}\cdot\frac{m_i^2}{2}.
\]
Combining with the lower bound on $\card{E_i}$ from Eq.~\eqref{eq:bound-E}, we obtain
\[
  \frac{m_i(m_i-1)}{2}
  -
  m_i\frac{4dC_{T,\delta}}{\epsilon_i^2}
  \le
  \Parens{1-\frac{1}{d}}\cdot\frac{m_i^2}{2}.
\]
Dividing by $m_i$, multiplying by $2d$, and rearranging then yields
\[
  m_i
  \le
  2d\Parens{\frac12+\frac{4dC_{T,\delta}}{\epsilon_i^2}}
  =
  d+\frac{8d^2C_{T,\delta}}{\epsilon_i^2}.
\]
We have originally assumed that $m_i\ge 2$, but the bound that we have just derived also holds when $m_i\le 1$ (because $d\ge 1$).

To complete the proof it suffices to sum up the upper bounds on $m_i$ across $i=1,\dots,n$:
\begin{align}
\notag
  \sum_{i=1}^n m_i
  &=
  \sum_{i=1}^n \BiggBracks{d+\frac{8d^2C_{T,\delta}}{\epsilon^2}\cdot\Parens{1/q^2}^{i-1}}
\\
\label{eq:sum}
  &\le
  nd + \frac{8d^2C_{T,\delta}}{\epsilon^2}\cdot\frac{1}{1-(1/q^2)}.
\end{align}
To bound $n$, we use the fact that $\alpha\le 1$, the inequality $\ln(1+x)\ge\frac{x}{1+x}$ (which holds for $x\ge 0$), and the fact that $d\ge 1$:
\begin{align*}
  n
&\le 1 +\log_q(2/\epsilon)
\\
&
 =   1+\frac{\ln(2/\epsilon)}{\ln\bigParens{1+\frac{1}{\sqrt{d}}}}
 \le 1+[\ln(2/\epsilon)]\cdot\frac{1+\frac{1}{\sqrt{d}}}{\frac{1}{\sqrt{d}}}
  =  1+(\sqrt{d}+1)\ln(2/\epsilon)
\\
&\le 2\ln 2 + 2\sqrt{d}\ln(2/\epsilon)
 \le 2\sqrt{d}\ln(4/\epsilon).
\end{align*}
Also,
\[
  1-\frac{1}{q^2}
  =
  1-\frac{1}{1+\frac{2}{\sqrt{d}}+\frac{1}{d}}
  \ge
  1-\frac{1}{1+\frac{2}{\sqrt{d}}}
  =
  \frac{\frac{2}{\sqrt{d}}}{1+\frac{2}{\sqrt{d}}}
  =
  \frac{2}{\sqrt{d}+2}
  \ge
  \frac{2}{3\sqrt{d}}.
\]
Plugging these back in Eq.~\eqref{eq:sum} yields
\begin{align}
\notag
  \sum_{i=1}^n m_i
  &\le
  2d^{1.5}\ln(4/\epsilon)
  +
  \frac{12d^{2.5}C_{T,\delta}}{\epsilon^2},
\end{align}
completing the proof of the main claim of the lemma.

The second claim holds vacuously when $T=0$, so assume that $T\ge 1$. If $C_{T,\delta}\ge\ln(2T)$, and using the fact that $(\ln x)\le x$ and $2\le3\ln2$, we can write
\[
2d^{1.5}\ln(4/\epsilon)
=
 d^{1.5}\ln(16/\epsilon^2)
\le
 \frac{16d^{1.5}}{\epsilon^2}
\le
 \frac{24(\ln 2)d^{1.5}}{\epsilon^2}
\le
 \frac{24d^{2.5}}{\epsilon^2}\cdot\ln(2T)
\le
  \frac{24d^{2.5}C_{T,\delta}}{\epsilon^2},
\]
which yields the first part of the second claim. The second part is immediate by plugging in $C_{T,\delta}=0$ in the main claim.

\subsection{Useful lemmas}

In this subsection we prove two lemmas that will be needed for the proofs of the main results (Theorems \ref{thm:regret_optimism} and \ref{thm:main_pac}) in Appendices~\ref{appendix:pf_main_regret} and~\ref{appendix:pf_main_pac}. They both rely on a standard technique of bounding a sum by a definite integral:

\begin{proposition}
\label{prop:int}
Let $f:\reals\to\reals$ be a non-increasing function and $T\ge 1$. Then
\[
  \sum_{t=1}^T f(t)
  \le
  f(1)+\int_1^T f(t)\mathrm{d}t.
\]
\end{proposition}
\begin{proof}
The proof is immediate by noting that $f(t)\le\int_{t-1}^t f(t)\mathrm{d}t$.
\end{proof}

In the lemmas below we write $\reals_+$ to denote $[0,+\infty)$.

\begin{lemma}\label{lemma::helper_regret_result}
Let $q_1, \dotsc, q_T$ be a sequence in $\reals_+$,
and let $\kappa: \reals_+\to\reals_+$ be a non-increasing function such that
for all $\epsilon  > 0$,
\begin{equation*}
\sum_{t=1}^T \mathbf{1}(q_t \geq \epsilon )   \leq \frac{  \kappa(\epsilon)}{\epsilon^2}.
\end{equation*}
Then, for any $\tau\ge 0$,
\begin{equation*}
\sum_{t=1}^T q_t \leq T\tau + 2\sqrt{  \kappa(\tau) T}.
\end{equation*}
\end{lemma}

\begin{proof}
 First, since we are only concerned with bounding the sum $\sum_t q_t$,
 we assume without loss of generality that the sequence is in descending order, i.e., $q_1 \ge \dotsb \ge q_T$. Then, for any $\tau\ge 0$,
\begin{equation}\label{eq:helper_lem_1}
\sum_{t=1}^T q_{t} = \sum_{t=1}^T q_{t}\mathbf{1}( q_{t} \leq \tau )+  \sum_{t=1}^T q_{t}\mathbf{1}( q_{t} > \tau) \leq T\tau + \sum_{t=1}^T q_{t}\mathbf{1}( q_{t} > \tau).
\end{equation}

Consider any $k$ such that $q_k>\tau$. Since the sequence $q_{1},\dotsc, q_{T}$ is non-increasing, we have
\[
k
  \leq \sum_{t=1}^T \mathbf{1}( q_{t} \geq q_{k})  \leq  \frac{  \kappa(q_k)}{q^2_{k}}  \leq \frac{  \kappa(\tau)}{q^2_{k}},
\]
where the last inequality follows by the monotonicity of $\kappa$. This in turn implies that $q_{k}\leq \sqrt{ \frac{  \kappa(\tau)}{k} }$. Therefore,
\begin{align}\label{eq:helper_lem_2}
\sum_{t=1}^T q_t \mathbf{1}(q_t > \tau)   &\leq \sum_{t=1}^T  \sqrt{ \frac{  \kappa(\tau)}{t} }.
\end{align}

By Proposition~\ref{prop:int},
\begin{equation}
\label{eq:helper_lem_3}
    \sum_{t=1}^T\frac{1}{\sqrt{t}}
    \le
    1+2\sqrt{T}-2\sqrt{1}
    <2\sqrt{T}.
\end{equation}
Combining Eqs.~\eqref{eq:helper_lem_1}, \eqref{eq:helper_lem_2} and~\eqref{eq:helper_lem_3},
we get
\[
    \sum_{t=1}^T q_t \leq T\tau +  2\sqrt{  \kappa(\tau)T}.
\]
which concludes the proof.
\end{proof}

\begin{lemma}\label{lemma::helper_regret_result_log}
Let $a>0$, let $q_1, \dotsc, q_T$ be a sequence of reals in $[0,a]$,
and let $\kappa: \mathbb{R}_+ \rightarrow \mathbb{R}_+$ be a non-increasing function such that
for all $\epsilon\in(0,a]$,
\begin{equation*}
\sum_{t=1}^T \mathbf{1}(q_t \geq \epsilon )   \leq  \kappa(\epsilon) \ln\left(\frac{a}{\epsilon}\right).
\end{equation*}
Then, for any $\tau\ge 0$,
\begin{equation*}
\sum_{t=1}^T q_t \leq T\tau + a[1+\kappa(\tau)]\exp\left( - \frac{1}{\kappa(\tau)}\right).
\end{equation*}
\end{lemma}

\begin{proof}
We follow a similar proof strategy as in Lemma~\ref{lemma::helper_regret_result} and start by bounding the sum $\sum_t q_t$. We assume without loss of generality that the sequence is in descending order, i.e., $q_1 \ge \dotsb \ge q_T$. Then, for any $\tau\ge 0$,
\begin{equation}
\label{eq:helper_lem_log1}
\sum_{t=1}^T q_{t} = \sum_{t=1}^T q_{t}\mathbf{1}( q_{t} \leq \tau )+  \sum_{t=1}^T q_{t}\mathbf{1}( q_{t} > \tau) \leq T\tau + \sum_{t=1}^T q_{t}\mathbf{1}( q_{t} > \tau).
\end{equation}

Consider any $k$ such that $q_k>\tau$. Then
\[
    k \leq \sum_{t=1}^T \mathbf{1}( q_{t} \geq q_{k})  \leq  \kappa(q_k) \ln\left( \frac{a}{q_k}\right) \leq  \kappa(\tau) \ln\left( \frac{a}{q_k}\right),
\]
where the last inequality follows by the monotonicity of $\kappa$ and the fact that $\ln(a/q_k)\ge 0$. This in turn implies that
$q_{k}\leq a\exp\bigParens{ -\frac{k}{\kappa(\tau)}}$.
Therefore,
\begin{align}
\label{eq:helper_lem_log2}
\sum_{t=1}^T q_t \mathbf{1}(q_t > \tau)   &\leq \sum_{t=1}^T  a\exp\left( -\frac{t}{\kappa(\tau)}\right).
\end{align}

By Proposition~\ref{prop:int},
\begin{align}
\notag
\sum_{t=1}^T\exp\left( -\frac{t}{\kappa(\tau)}\right)
   &
\le
   \exp\Parens{-\frac{1}{\kappa(\tau)}}
   -\kappa(\tau)
   \Parens{
   \exp\Parens{-\frac{T}{\kappa(\tau)}}
   -
   \exp\Parens{-\frac{1}{\kappa(\tau)}}
   }
\\
\label{eq:helper_lem_log3}
&
\le
   [1+\kappa(\tau)]\exp\Parens{-\frac{1}{\kappa(\tau)}}.
\end{align}
Combining Eqs.~\eqref{eq:helper_lem_log1}, \eqref{eq:helper_lem_log2} and~\eqref{eq:helper_lem_log3},
we get
\begin{equation*}
\sum_{t=1}^T q_t \leq T\tau + a[1+\kappa(\tau)]\exp\left( - \frac{1}{\kappa(\tau)}\right),
\end{equation*}
which concludes the proof.
\end{proof}

\subsection{Proof of Theorem \ref{thm:regret_optimism}} \label{appendix:pf_main_regret}

Throughout the proof we use the shorthand $d_{\epsilon}=\dmon_\score(\Z, \alpha,\epsilon)$,
so $d=d_{1/T}$.
The proof proceeds by applying Lemmas~\ref{lemma::helper_regret_result} and~\ref{lemma::helper_regret_result_log} to the bounds on the number of bad queries from Lemma~\ref{lem:few_bad_queries}. Specifically, let $q_t = [\alpha - \rho(z_t \given z^*)]_+$ denote the suboptimality of each query $z_t$ made by the algorithm. Then, for any $\epsilon>0$, the number of $\epsilon$-bad queries can be written as $\sum_{t=1}^T \mathbf{1}(q_t \ge \epsilon)$.

First consider the case $C_{T,\delta}\ge\ln(2T)$. By Lemma~\ref{lem:few_bad_queries}, with probability at least $1-\delta$, the number of $\epsilon$-bad queries is at most $36d_\epsilon^{2.5}C_{T,\delta}/\epsilon^2$. Setting
$\kappa(\epsilon)=36d_\epsilon^{2.5}C_{T,\delta}$, we apply Lemma~\ref{lemma::helper_regret_result}, with $\tau=1/T$, to obtain that with probability at least $1-\delta$,
\begin{align*}
  \mathrm{Regret}(T,\alpha)
  \le
  \sum_{t=1}^T q_t
  \le
  1+12 d^{1.25}\sqrt{C_{T,\delta}T}.
\end{align*}

If $C_{T,\delta}=0$, then by Lemma~\ref{lem:few_bad_queries}, the number of $\epsilon$-bad queries is at most $2d_{\epsilon}^{1.5}\ln(4/\epsilon)$. Setting $a=4$ and
$\kappa(\epsilon)=2d_\epsilon^{1.5}$, we apply Lemma~\ref{lemma::helper_regret_result_log}, with $\tau=1/T$, to obtain
\begin{align*}
  \mathrm{Regret}(T,\alpha)
  \le
  \sum_{t=1}^T q_t
  \le
  1+
  4(1+2d^{1.5})\exp\Parens{-\frac{1}{2d^{1.5}}}
  \le
  1+12d^{1.5},
\end{align*}
completing the proof.

\subsection{Proof of Theorem \ref{thm:main_pac}}\label{appendix:pf_main_pac}

First, we consider the deterministic setting.
    By Lemma \ref{lem:few_bad_queries}, at most $2d^{1.5}\ln(4/\epsilon)$ of queries issued by $\Alg$ are $\epsilon$-bad. Setting $T>2d^{1.5}\ln(4/\epsilon)$ implies that at least one query is not $\epsilon$-bad. Thus, returning $\hat{z}$ for which the observed reward is the largest guarantees that $\score(\zhat\given z^*) \ge \alpha - \epsilon$, as needed.

    Next, we prove the result for the case $C_{T,\delta}\ge\ln(2T)$.
    By Lemma \ref{lem:few_bad_queries}, with probability at least $1-\delta/2$, there are at most $\frac{16}{9\epsilon^2}\cdot36d^{2.5}(C_{T,\delta/2})$ queries that are $3\epsilon/4$-bad. Setting
    $T\ge64d^{2.5}(C_{T,\delta/2})/\epsilon^2$, implies that at least half of the queries are not $3\epsilon/4$-bad. In the remainder of the proof, we only consider the high-probability event in which this is the case.

    For $n_1=\lceil\log_2(4/\delta)\rceil$ the probability that all $n_1$ samples are $3\epsilon/4$-bad is at most $(1/2)^{n_1} \le \delta/4$.

    For $n_2 = \lceil 128\ln(8n_1/\delta)/\epsilon^2\rceil$, by applying Hoeffding's inequality and union bound over each of the  $n_1$ rounds we get that with probability at most $\delta/4$ there is some index $\ell \le n_1$ for which $|\Bar{r}_{t_{\ell}} - \rho(z_{t_\ell}\given z^*)| > \epsilon/8$.

    Overall, with probability at least $1-\delta$ we get that there is at least one index $j$ of the $n_1$ sampled indices that is not $3\epsilon/4$-bad, and that $|\Bar{r}_{t_{\ell}} - \rho(z_{t_{\ell}}\given z^*)| \le \epsilon/8$ for all $\ell=1,\dotsc,n_1$. Therefore,
    \[
    \Bar{r}_{t_j} \ge \rho(z_{t_j}\given z^*) -\epsilon/8 \ge \alpha -3\epsilon/4-\epsilon/8=\alpha - 7\epsilon/8.
    \]
    For all indices $k$ that are $\epsilon$-bad we have
    \[
    \Bar{r}_{t_k} \le \rho(z_{t_k}\given z^*) + \epsilon/8 < \alpha - \epsilon + \epsilon/8 = \alpha  - 7\epsilon/8.
    \]
    Thus, for all of the $\epsilon$-bad queries we have $\Bar{r}_{t_k}<\Bar{r}_{t_j}$, and so
    Algorithm~\ref{alg:MAIN12} will not return any of the $\epsilon$-bad queries, because it is choosing the index with maximum value of $\Bar{r}_{t_\ell}$. In other words, the returned query $z_{t_{\hat{\ell}}}$ satisfies
    \[
    \rho(z_{t_{\hat{\ell}}}\given z^*) \ge \alpha - \epsilon.
    \]

\section{Missing proofs of Section \ref{sec:sq}} \label{appendix:sq}

First we discuss the connection between our SQ setting and the SQ model of~\citet{kearns1998efficient}. We focus on two aspects in which they appear to differ and explain why these models are equivalent.

\paragraph{Correlational \textit{vs} general statistical queries.}
The restriction of the SQ model in which the oracle may only output the approximate correlation between a query and the target function, termed \textit{correlational statistical query} (CSQ), was studied by \citet{bshouty2002using}. The CSQ oracle  can be viewed as providing something akin to a negative distance between the query and the target.
This is equivalent to the \emph{learning by distances} framework of \citet{ben1995learning}, who defined their model independently of~\citet{kearns1998efficient}.
\citet{bshouty2002using} showed that an arbitrary statistical query can be answered by asking two SQs that are independent of the target and two CSQs. That is, in the distribution-dependent learning model (i.e., when the learner has access to the distribution over $\X$), correlational queries can simulate general queries.

\paragraph{Adversarial \textit{vs} statistical noise.}
The setting we consider in this work assumes stochastic query responses, similar to
several previous works~\citep{feldman2017statistical,yang2005new,ben1998learning}.
On the other hand, the original SQ model~\citep{kearns1998efficient} assumed that
the query oracle can respond with an adversarial (rather than statistical) noise, up to a pre-specified tolerance parameter $\tau > 0$. The previous works have shown that the two noise models are equivalent~\citep{feldman2017statistical,yang2005new,ben1998learning}

\subsection{Proof of Proposition \ref{lem:ssqd_to_newdim}}
\label{sec:proof:lem:ssqd_to_newdim}

We first prove the first inequality of Eq.~(\ref{eq:lem:ssqd_to_newdim:1}).
Let $\epsilon>0$ and let $d=\dsq(\epsilon)$.
Then there exists a sequence $h_1,\dotsc,h_d \in \H$ satisfying both conditions of Definition~\ref{def:ssqd}.
Let $d'$ be equal to the leftmost expression
of Eq.~(\ref{eq:lem:ssqd_to_newdim:1}).
We aim to show $\drho(\epsilon)\geq d'$.
Note that $d'\leq d$.

Let $c$ be the midpoint between $c_{min} = \min_{i<j} \langle h_i, h_j \rangle$ and $c_{max} = \max_{i<j} \langle h_i, h_j \rangle$. Then $c \le 1 - \epsilon$. Moreover, for all $i\neq j$,
\[
|\langle h_i, h_j \rangle - c| \le \frac{1}{2} | c_{max} - c_{min}| \le \frac{1}{2d} \le \frac{\epsilon}{\sqrt{d'}}
\]
where the last inequality follows from our choice of $d'$
(which ensures $d'\leq 4(d\epsilon)^2$).
Thus, $h_1,\ldots,h_{d'}$, the first $d'$ elements of the original sequence of hypotheses, satisfy
Definition~\ref{def:new_dim}, proving
the claim.

We prove the first inequality of Eq.~(\ref{eq:lem:ssqd_to_newdim:2}) in a similar way.
Let us re-define $d = \drho(4\epsilon)$
and let $d'$ be equal to the leftmost expression of
Eq.~(\ref{eq:lem:ssqd_to_newdim:2}).
As before, $d'\leq d$.
Then there exists a sequence $h_1,\dotsc,h_d \in \H$
satisfying the conditions of Definition~\ref{def:new_dim}
for some $c \le 1 - 4\epsilon$.
Then for all $i\neq j$,
$\langle h_i, h_j \rangle \le c  + \frac{4\epsilon}{\sqrt{d}} \le 1 - \epsilon$, since $d \ge 2$.
Moreover, for all $i\neq j$ and $i'\neq j'$,
\begin{align*}
    |\langle h_i, h_j \rangle - \langle h_{i'}, h_{j'} \rangle| = |\langle h_i, h_j \rangle - c + c - \langle h_{i'}, h_{j'} \rangle|
    \le \frac{8\epsilon}{\sqrt{d}}
    \le \frac{1}{d'},
\end{align*}
with the last inequality following from our choice of $d'$.
Thus, $h_1,\ldots,h_{d'}$, the first $d'$ hypotheses in the original sequence, satisfy
Definition~\ref{def:ssqd}.

The second inequality of
Eq.~(\ref{eq:lem:ssqd_to_newdim:2}),
now follows from the first inequality of
Eq.~(\ref{eq:lem:ssqd_to_newdim:1}),
since if
the second inequality of
Eq.~(\ref{eq:lem:ssqd_to_newdim:2})
does not hold then the leftmost expression of
Eq.~(\ref{eq:lem:ssqd_to_newdim:1})
must be at least $\drho(\epsilon)$, a contradiction.
Likewise,
the second inequality of
Eq.~(\ref{eq:lem:ssqd_to_newdim:1}),
now follows from the first inequality of
Eq.~(\ref{eq:lem:ssqd_to_newdim:2}).

\subsection{Lower bound setting}\label{appendix:sq_lb}

\begin{definition}[SQ oracle (adversarial)]
  Let $D$ be the input distribution over the domain $X$. For a {\em tolerance} parameter $\tau > 0$, $\mathcal{O}^{adv}(\tau):= \mathcal{O}^{adv}_{D,h^*}(\tau)$ oracle is the oracle that for any query function $h \in \H$, returns a value $v \in \left[ \mu - \tau, \mu+ \tau\right]$, where $\mu = \E_{x\sim D}[h(x)h^*(x)]$.
\end{definition}

\begin{definition}[Sample oracle (statistical)]
  Let $D$ be the input distribution over the domain $X$. The Sample oracle $\mathcal{O}:= \mathcal{O}_{D,h^*}$ oracle is the oracle  that given any function $h \in \H$,
  takes an independent random sample $x$ from $D$ and returns the value $v = h(x)h^*(x)$.
\end{definition}

We will need the following results for our proof. The first is a reduction from an adversarial noise oracle to a statistical one. Specifically, consider the learning setting defined in Section \ref{sec:setting},
for a sample oracle $\mathcal{O}$. Let $\Alg$ be a (possibly randomized) algorithm for that setting. The following theorem shows a simulation of $\mathcal{O}$ via $\mathcal{O}^{adv}$ the SQ oracle.

\begin{theorem}[\cite{feldman2017statistical}, Theorem 3.13]
\label{th:unbiased-from-vstat}
Assume that $\Alg$ outputs a $\epsilon$-approximation to $h^*$ with probability at least $\delta$, using $m$ samples from
$\O$.  Then, for any $\delta' \in (0,1/4]$, there exists a SQ algorithm $\Alg'$ that uses at most $m$ queries to $\mathcal{O}^{adv}({\delta'}^2/m)$ and outputs an $\epsilon$-approximation to $h^*$ with  probability at least
$\delta - \delta'$.
\end{theorem}

Their result is obtained by simulating  $\Alg$ using  $\O^{adv}$ as follows: for any query of $\Alg$ to  $\O$, the response of $\O^{adv}$ to that query is used as bias for a coin flip, which is then given to the learner as the simulated outcome of $\O$. They then prove that the true $m$ samples of $\O$ and the simulated coin flips are statistically
close by bounding their distributional distance. This implies that the success probability of $\Alg'$, the simulated algorithm, is not much worse than that of $\Alg$, the original algorithm.

We note that the result originally stated in \cite{feldman2017statistical} differs from Theorem \ref{th:unbiased-from-vstat} above in two ways. First, it reduces to a \textit{variant} of $\mathcal{O}^{adv}(\tau)$ with a tolerance $\tau' \in [\tau, \sqrt{\tau}]$. Thus, it holds for $\mathcal{O}^{adv}(\tau)$ as well. Second, it is phrased in a more general setting of search problems over distributions, which captures the SQ model, as detailed in \cite{feldman2017statistical}, Section 6.

The second result that is needed for our proof is the following lower bound due to \cite{szorenyi2009characterizing}.

\begin{theorem}[\cite{szorenyi2009characterizing}, Theorem 8]\label{thm:szorenyi_lb}
Let $\epsilon>0$, and let $\H \subseteq \{\pm 1\}^\X$ be a hypothesis space with strong SQ dimension $\dsq\coloneqq\dimSQ(\H, 2\epsilon) \ge 3$ (see Definition \ref{def:ssqd}).
Then for any SQ algorithm $\Alg$ using $m$ queries to $\O^{adv}(\tau)$ with tolerance
$\tau \ge 2/\sqrt{\dsq}$,
there exist $h^* \in \H$ such that if $\Alg$ outputs an $\epsilon$-approximation to $h^*$, then $m > \dsq\tau^2/3$.
\end{theorem}

\subsubsection{Proof of Theorem \ref{thm:lb}}

Set $\delta = 2/3$. Let $D$ be a distribution over $\X$. Assume towards contradiction that
there exists a learning algorithm $\Alg$ such that for any $h^* \in \H$, given oracle access to $\mathcal{O}:= \mathcal{O}_{D,h^*}$ and using  $m\le\sqrt[3]{\dsq}/12$ samples from
$\O$, the algorithm $\Alg$ outputs an $\epsilon$-approximation to $h^*$ with probability at least $\delta$.

We then apply Theorem \ref{th:unbiased-from-vstat} for $\delta' =  \delta/2$ to simulate the
algorithm using $\O^{adv}:= \O^{adv}_{D,h^*}$. The resulting algorithm uses  $m \le \sqrt[3]{\dsq}/12$ queries to $\O^{adv}(\tau)$
for $\tau ={\delta'}^2/m  > 4/(3\sqrt[3]{\dsq}) \ge 2/\sqrt{\dsq}$ and has success probability of at least $\delta - \delta' = \delta/2 > 1/3$. By Theorem \ref{thm:szorenyi_lb} we obtain a contradiction, as $m > \dsq\tau^2/3 > \sqrt[3]{\dsq}/2$.

\section{Missing proofs of Section \ref{sec:bandits}}\label{appendix:bandits}

\subsection{Proof of Theorem \ref{theorem::upper_eluder}}\label{appendix::upper_eluder_proof}

Let $d=\dmon_{\score}( \Z, \alpha, 3\epsilon/2)$. Note that the eluder dimension is always at least $1$, so the theorem trivially holds if $d\le 9$. In the remainder of the proof assume that $d\ge10$.

From the definition of the monotonic dissimilarity dimension, there exists $\tau\ge3\epsilon/2$ such that $d=\d(\Z,\alpha,\tau)$. Let $(f_1, a_1),\dotsc,(f_d, a_{d})$ be a sequence satisfying the dimension conditions for $\tau$. We will show that
the first $\Ceil{d/9}$ elements of this sequence also satisfy the conditions of the eluder dimension for some $\epsilon'\ge\epsilon$.
Specifically, we will show that there is some $\epsilon' \ge \epsilon$ such that every element $a_j$ with $j\le\Ceil{d/9}$ in the sequence above is $\epsilon'$-independent of its predecessors. That is, we will show that for every such element $a_j$, there exists a pair of functions $f,f' \in \F$ that satisfy
\[
  \sqrt{ \sum_{i=1}^{j-1} \bigParens{f(a_i) - f'(a_i)}^2 } \leq \epsilon',
\]
yet it also holds that $f(a_j) - f'(a_j) > \epsilon'$.

By definition of the dissimilarity dimension, there exists $c \le \alpha -\tau$
such that for all $i<j$,
    \begin{equation} \label{eq:lemma_eluder_pf_eq_1}
    \BigAbs{f_j(a_i) - c} = \BigAbs{\score\bigParens{(f_i,a_i)\bigGiven(f_j,a_j)} - c} \le \frac{\tau}{\sqrt{d}}.
    \end{equation}
Then, by the triangle inequality,
    \begin{equation}\label{eq:lemma_eluder_pf_eq_2}
    \bigAbs{f_j(a_i) -  f_{j+1}(a_i)}
    =
    \bigAbs{f_j(a_i) - c +c -  f_{j+1}(a_i)}
    \le
    \bigAbs{f_j(a_i) - c} + \bigAbs{f_{j+1}(a_i) - c}
    \le
    \frac{2\tau}{\sqrt{d}}.
    \end{equation}
Therefore,
    \begin{equation}\label{eq:lemma_eluder_pf_eq_3}
             \bigParens{f_j(a_i) -  f_{j+1}(a_i)}^2  \le \frac{4\tau^2}{d},
    \end{equation}
and so for all $j \le\Ceil{d/9}$ it holds that,
    \begin{equation} \label{eq:lemma_eluder_pf_eq_4}
         \sum_{i=1}^{j-1}(f_j(a_i) -  f_{j+1}(a_i))^2 <\frac{4\tau^2}{9}.
    \end{equation}

Next, recall that for all $j \le d$ we have
    $f_j(a_j) \ge\alpha\ge c+\tau$ and $f_{j+1}(a_j) \le c + \frac{\tau}{\sqrt{d}}$.
        Thus,
     \begin{equation} \label{eq:lemma_eluder_pf_eq_5}
    f_j(a_j) - f_{j+1}(a_j)\ge c+\tau - c - \frac{\tau}{\sqrt{d}} >\frac{2\tau}{3},
    \end{equation}
where the last inequality holds for $d \ge 10$.
    Overall, Eqs.~\eqref{eq:lemma_eluder_pf_eq_4} and~\eqref{eq:lemma_eluder_pf_eq_5} then demonstrate that for $\epsilon'=2\tau/3\ge\epsilon$, the element $a_j$ is $\epsilon'$-independent of its predecessors, finishing the proof.

\subsection{Proof of Theorem \ref{theorem::linear_dim_bound}} \label{appendix:linear_dim_bound_pf}

Our proof uses the following result on ranks of perturbed identity matrices (see \citep[Lemma 2.2]{alon2009perturbed}):
\begin{lemma}
\label{lemma:perturbed}
Let $\Ab\in\reals^{d\times d}$ be a symmetric matrix such that $A_{ii}=1$ for all $i$ and $\abs{A_{ij}}\le1/\sqrt{d}$ for all $i\ne j$. Then $\rank(\Ab)>d/2$.
\end{lemma}

    The proof begins by constructing a matrix $\Mb$ whose entries are derived from the evaluation values of elements that satisfy the dimension condition. Then we bound the rank of $\Mb$ from above as well as from below. The lower bound is expressed in terms of the dimension $d = \d_\score(\Z, \alpha,  \epsilon)$ while the upper bound is expressed in terms of $n$. Combining the bounds then yields the result of the theorem.

    \textbf{Construction of $\Mb$.}
    Let  $(f_{\boldsymbol{\theta}_1}, \ab_1),\dotsc,(f_{\boldsymbol{\theta_d}}, \ab_d)$ denote the alternatives that satisfy the dimension conditions, with respect to some value
    $c$ such that $c \le \alpha - \epsilon$ (see Definition \ref{def:new_dim}). Define $\Mb$ to be the $d \times d$ matrix with entries $M_{ij} = \langle \boldsymbol{\theta}_i, \ab_j \rangle-c$ for $i,j \le d$. Note that all diagonal entries of $\Mb$ are at least $\alpha-c\ge\epsilon$, and all other entries are in $\bigBracks{-\frac{\epsilon}{\sqrt{d}}, \frac{\epsilon}{\sqrt{d}}}$.

    \textbf{Upper bound on $\rank(\Mb)$.}
    Let $\Kb\in\reals^{d\times d}$ be the matrix of inner products, $K_{ij}=\langle \boldsymbol{\theta}_i, \ab_j \rangle$, and let $\Ub$ be the Gram matrix for the set of vectors $\boldsymbol{\theta}_1,\dotsc,\boldsymbol{\theta}_d,\ab_1,\dotsc,\ab_d$. Then $\Ub$ is a $2d \times 2d$ matrix
    of the rank at most $n$, because the vectors are of the dimension $n$ (see, e.g., \citep[Theorem 7.2.10]{horn2012matrix}), and $\Kb$ is a submatrix of $\Ub$, so $\rank(\Kb) \le \rank(\Ub) \le n$. Moreover,
    $\Mb=\Kb-c\one\one^\top$, where $\one$ is the all-ones vector in $\reals^d$. Therefore, by subadditivity of rank,
    \begin{equation}
    \label{eq:M:upper}
       \rank(\Mb)
       =
       \rank(\Kb-c\one\one^\top)
       \le
       \rank(\Kb) + \rank(-c\one\one^\top)
       \le
       n+1.
    \end{equation}

    \textbf{Lower bound on $\rank(\Mb)$.}
    Let $\Db\in\reals^{d\times d}$ be the diagonal matrix
    with entries $D_{ii}=1/\sqrt{M_{ii}}$. Since $M_{ii}\ge\epsilon>0$, we have $0<D_{ii}\le1/\sqrt{\epsilon}$. Consider the matrix
    $\Mb'=\Db\Mb\Db$.
    Then
    \begin{equation}
    \label{eq:M'}
       \rank(\Mb')=\rank(\Db\Mb\Db)=\rank(\Mb),
    \end{equation}
    because the matrix $\Db$ is non-singular (see \citep[Section 0.4.6(b)]{horn2012matrix}).
    Furthermore, matrix $\Mb'$ satisfies $M'_{ii}=1$ for all $i$ and
    \[
      \smash[t]{
        \abs{M'_{ij}}=\abs{D_{ii} M_{ij} D_{jj}}
        \le
        \frac{1}{\sqrt{\epsilon}}\cdot
        \frac{\epsilon}{\sqrt{d}}\cdot
        \frac{1}{\sqrt{\epsilon}}
        =
        \frac{1}{\sqrt{d}}
      }
    \]
    for all $i\ne j$.
    Consider the symmetric matrix $\Sb=(\Mb'+(\Mb')^\top)/2$. Then, we also have $S_{ii}=1$ for all $i$, and $\abs{S_{ij}}\le1/\sqrt{d}$ for all $i\ne j$. Thus, by Lemma~\ref{lemma:perturbed}, $\rank(\Sb)>d/2$. Moreover, by the subadditivity of the rank
    \begin{equation}
    \label{eq:M:lower}
      d/2<\rank(\Sb)\le\rank(\Mb'/2)+\rank\bigParens{(\Mb')^\top\!/2}=2\rank(\Mb').
    \end{equation}
    Combining Eqs.~\eqref{eq:M:lower}, \eqref{eq:M'} and \eqref{eq:M:upper}, we therefore obtain
    \[
      d/2<2\rank(\Mb')=2\rank(\Mb)\le2n+2,
    \]
    and so $d<4n+4$. Since $d$ is an integer, we must have $d\le 4n+3$.

    In the special case that $\alpha=1$, we have $\langle \boldsymbol{\theta}_i, \ab_i \rangle\ge 1$ for all $i\le d$, which is only possible when $\boldsymbol{\theta}_i=\ab_i$ for all $i\le d$. As a result, the matrices $\Mb$ and $\Mb'$ are both symmetric, and thus
    $\Sb=\Mb'$ and
    \[
      d/2<\rank(\Sb)=\rank(\Mb')=\rank(\Mb)\le n+1,
    \]
    implying that $d\le 2n+1$.

\subsection{Proof of Theorem \ref{theorem::glm_dim_bound}} \label{appendix:glm_dim_bound_pf}
Using an existing bound on the eluder dimension for GLM bandits (\cite{russo2013eluder}, Proposition 7), and the fact that our dimension is bounded by the eluder dimension (Theorem \ref{theorem::upper_eluder}) the result follows.

\subsection{Proof of Theorem \ref{theorem::relu_dim_bound}} \label{appendix:relu_dim_bound_pf}

Denote $d = \d_{\score}(\Z^{\relu}_b, 1-b,  \epsilon)$. Notice that since $b < 1$, for any $\thetab,  \ab \in \B_n$ such that $\thetab \neq  \ab$ it holds that $f_{\thetab,b}(\ab) < f_{\thetab,b}(\thetab) = 1-b$. Let $(f_{\thetab_1,b}, \thetab_1), \dotsc, (f_{\thetab_d,b}, \thetab_d)$ be a sequence of elements satisfying the dimension definition, with respect to a corresponding scalar $c \le 1-b-\epsilon$. Since the evaluation is symmetric for ReLU functions, we can view this sequence as a set, and denote $U = \{\thetab_1,\dotsc\thetab_d\}$. In addition, note that by the dimension definition, for all $\thetab \in U$, $\|\thetab\| = 1$.

We start by proving an upper bound on $d$. Assume $d \ge 9$.
First, consider the case $c \le \epsilon/3$. Let $U_0$ be any subset of the unit sphere such that for all $\thetab\neq \thetab'$ in $U_0$ it holds that $\langle \thetab,\thetab' \rangle \le b+ 2\epsilon/3$. Observe that for all such $\thetab\neq \thetab'$ we have $f_{\thetab,b}(\ub') = f_{\thetab',b}(\thetab) \in [0,2\epsilon/3]$.   Thus, we get that $d \le |U_0|$.

A standard sphere covering argument shows that the size of such a set is upper bounded as follows. The $\delta$-covering number of the unit sphere is at most $(3/\delta)^n$ (\cite{vershynin2018high},  Cor. 4.2.13). Thus, there are at most $(3/\delta)^n$  points such that each pair $\thetab\neq \thetab'$  satisfies $\| \thetab - \thetab'\| \ge \delta$, or equivalently
$\langle \thetab,\thetab' \rangle \le 1 - \delta^2/2$.
By setting $\delta = \sqrt{2(1-b-2\epsilon/3)}$  we get that $|U_0| \le (\frac{3}{\delta})^n \le (\frac{3}{2\sqrt{2(\epsilon-2\epsilon/3)}})^n \le (4/\sqrt{\epsilon})^n$, which yields the desired bound.

Now, consider the case $c > \epsilon/3$. In this case, for all $i\neq j$, we have that $f_{\thetab_j,b}(\thetab_i) \ge c - \frac{\epsilon}{\sqrt{d}} \ge c - \epsilon/3 > 0$, and so $\langle \thetab_i, \thetab_j \rangle > b$. Let $c' = c+b$. Thus, it must also hold that, $|\langle \thetab_i, \thetab_j \rangle -c'| \le \frac{\epsilon}{\sqrt{d}}$ for all $i\neq j$. Note that $c' < 1-\epsilon$. Then, by applying Lemma \ref{theorem::linear_dim_bound} we get that
$d$ is upper bounded by $2n+4$.  Overall, the bound in the claim holds.

Next, we show a lower bound on $d = \d_{\score}(\Z^{\relu}_{1-\epsilon}, \epsilon,  \epsilon)$, by lower bounding the size of the set $U$ defined above. We now apply a sphere \emph{packing} argument, which shows that there exists such a set $U$ with size $|U| \ge (1/2\epsilon)^{n/2}$. We follow a similar argument as above. Specifically, the $\delta$-packing number of the unit sphere is at most $(1/\delta)^n$ (\cite{vershynin2018high},  Cor. 4.2.13).  By plugging in $\delta = \sqrt{2\epsilon}$, yields the desired bound.

\subsection{Proof of Proposition~\ref{prop::separation_eluder_dissimilarity}}\label{appendix:separation_eluder_dissimilarity_pf}

We start with an auxiliary lemma:

\begin{lemma}[\newdim subadditivity]\label{lemma::subaditivity_newdim}
Let $\Z_1$ and $\Z_2$ be two sets, and let $\alpha\in\reals$ and $\epsilon > 0$. Denote $\Z = \Z_1 \cup \Z_2$ and let $\score : \Z \times \Z \rightarrow \reals$ be an evaluation function. Then,
\begin{equation*}
\d_\score(\Z, \alpha, \epsilon) \leq \d_\score(\Z_1, \alpha, \epsilon) + \d_\score( \Z_2, \alpha, \epsilon)
\end{equation*}
and
\begin{equation*}
\dmon_\score(\Z, \alpha, \epsilon) \leq \dmon_\score(\Z_1, \alpha, \epsilon) + \dmon_\score( \Z_2, \alpha, \epsilon).
\end{equation*}
\end{lemma}

\begin{proof} Let $z_1, \dotsc, z_d \subseteq \Z$ such that there exists $c\le \alpha-\epsilon$ with $|\score(z_i | z_j) -c  | \leq \frac{\epsilon}{\sqrt{d}} $ for all $i < j$, and $\score(z_i|z_i) \ge \alpha$. Let $I_1, I_2 \subseteq [d]$ be disjoint sets of indices with $\{ z_i \}_{i \in I_1} \subseteq \Z_1$ and $I_2 = [d] \setminus I_1$. Consider the sub-sequence $z_{\ell_1}, \dotsc, z_{\ell_{|I_1|}}$ ordered by appearance in $z_1, \dotsc, z_d$ of elements in $I_1$. By definition for all $i<j$,
$$|\score( z_{\ell_i}| z_{\ell_j}) - c | \leq \frac{\epsilon}{\sqrt{d}} $$
Therefore the sequence $z_{\ell_1}, \dotsc ,z_{\ell_{|I_1|}} $  satisfies $|\score( z_{\ell_i}| z_{\ell_j}) - c | \leq \frac{\epsilon}{\sqrt{|I_1|}}$ for all $i<j$ and therefore $|I_1| \leq \d_\score(\Z_1, \alpha, \epsilon)$. The same logic implies $|I_2| \leq \d_\score(\Z_2, \alpha, \epsilon)$.

To get the monotonic version, note that if $\epsilon^* = \arg\max_{\epsilon' \ge \epsilon}\d_\score( \Z, \alpha, \epsilon')$,
\begin{align*}
\dmon_\score( \Z, \alpha, \epsilon) &= \d_\score( \Z, \alpha, \epsilon^*) \\
&\le \d_\score( \Z_1, \alpha, \epsilon^*) +\d_\score( \Z_2, \alpha, \epsilon^*) \\
&\le \dmon_\score( \Z_1, \alpha, \epsilon) +\dmon_\score( \Z_2, \alpha, \epsilon)
\end{align*}
where the first inequality is by the first statement of the lemma which was proved above, the next inequality is by definition of the monotonic dimension, and the last inequality follows by $\epsilon \le \epsilon^*$.
\end{proof}

We can now construct the classes that demonstrate the separation of the eluder and \newdim dimension.

{We consider two overlapping semicircles}, indexed by $j\in\{0,1\}$, and defined as
\[
U_0 = \BigBraces{(\cos x,\sin x):\:
                 x\in\Bigl(-\frac{\pi}{2},\frac{\pi}{2}\Bigr)},
\text{ and }
U_1 = \BigBraces{(\cos x,\sin x):\:
                 x\in\bigl(0,\pi\bigr)}.
\]
For each $j\in\{0,1\}$, and any $N\in\mathbb{N}$ and $\epsilon > 0$, we define the function class
\[
  \F_{j,N,\epsilon}\coloneqq
  \BigBraces{f_{\vb,S,\sigma}:\:
       \vb\in\C\setminus U_j,\,
       S\subseteq U_j,\,
       \card{S}=N,
       \sigma\in\{\pm\epsilon\}^{S}
  },
\]
containing functions
\begin{equation}
  f_{\vb,S, \sigma} (\ab) = \begin{cases}
        0 &\text{if } \ab \in U_j \setminus S,\\
        \sigma(\ab) &\text{if } \ab \in S, \\
        \langle \vb, \ab \rangle &\text{if $\ab\in\C\setminus U_j$}.
        \end{cases}
\end{equation}
In words, the functions in the class $\F_{j,N,\epsilon}$ are linear outside of the semicircle $U_j$,
and zero in the semicircle $U_j$, except for a set of size $N$, where they can take any combination of values $+\epsilon$ and $-\epsilon$. For any $N\in\mathbb{N}$ and $\epsilon>0$, we define
the class $\F_{N,\epsilon}\coloneqq\bigcup_{j\in\{0,1\}}\F_{j,N,\epsilon}$ and show that
this class has a constant dissimilarity dimension but its eluder dimension is at least $N$.

Finally, consider the action set $\A=\C$ and the function class $\F_{N,\epsilon}$ as defined above. Let $\Z_{N, \epsilon}=\F_{N, \epsilon}\times\A$, $\score=\scoreBandits$ and $\epsilon \in (0,1/2)$. We how show that  $\dimE(\F_{N,\epsilon},\epsilon)\geq N$, but
$\d_{\score}(\Z_{N, \epsilon}, 1, \epsilon) \leq 16$.
First we prove the following lower bound on the eluder dimension of $\F_{j, N, \epsilon}$.

\begin{lemma}\label{lemma::lower_bound_star_eluder_special_class}
The eluder dimension of $\F_{j, N, \epsilon}$ satisfies $\dimE(\F_{j, N, \epsilon}, \epsilon) \geq N$ for all $j \in \{0,1\}$.
\end{lemma}

\begin{proof} Let $\ab_1, \dotsc, \ab_N$ be an arbitrary set of points in $U_j$ and consider the functions $\{f_i\}_{i=1}^{N+1} \subseteq \F_{j, N, \epsilon, S, \vb}$ that for all $i, i' \le N$ satisfy:
\begin{equation*}
f_i(\ab_{i'}) = \begin{cases}
                 \epsilon &\text{if } {i'} \neq i\\
                - \epsilon &\text{if } {i'} = i,
            \end{cases}
\end{equation*}
and $f_{N+1}(\ab_i) = \epsilon$ for all $i \leq N$. We now show that for all $i \leq N$, the action $\ab_i$ is $\epsilon$-independent of $\ab_1, \dotsc, \ab_{i-1}$ with respect to $\F_{j, N, \epsilon, S, \vb}$. This holds since $\sqrt{\sum_{j=1}^{i-1} (f_i(\ab_j) - f_{N+1}(\ab_j))^{2 }  } = 0 $ while $|f_i(\ab_i) - f_{N+1}(\ab_i)| = 2\epsilon > \epsilon$. This finalizes the proof.
\end{proof}

\begin{lemma}\label{lemma::dissimilarity_dimension_upper_bound}
Denote $\Z_{N, \epsilon} = \F_{N, \epsilon} \times \A$ and let $\epsilon \in (0,1/2)$. Then,  $\d_\rho(\Z_{N, \epsilon}, 1, \epsilon) \leq 16$.
\end{lemma}

\begin{proof}
Denote by $\Z_{j,N, \epsilon} = \F_{j,N, \epsilon} \times \A$ for all $j \in \{0,1\}$. We start by showing that for all $j \in \{ 0, 1\}$, $\d_\rho(\Z_{j,N, \epsilon}, 1, \epsilon) \leq 8$.  Let $z_1, \dotsc, z_d$ be a maximal sequence certifying the dissimilarity dimension $\d_\rho(\Z_{N, \epsilon}, 1, \epsilon) \geq d$, i.e., it holds that,
\begin{equation*}
 |   \score( z_i | z_{i'}) -c |  \leq \frac{\epsilon}{\sqrt{d}} \text{ for all } i <{i'}, \qquad \text{while } \score( z_i | z_i ) \geq 1.
\end{equation*}
Since $\score(z_i | z_i ) \geq 1 $, it must be the case that $z_i = (f_{\vb_i, S_i, \sigma_i}, \vb_i)$ with $ \vb_i \not\in U_j$ (since otherwise the self evaluations would be strictly less than $1$).  This implies that $\score( z_i |z_j ) = \langle \vb_i, \vb_j\rangle$ for all $i< j$. Consequently, the score evaluations of all $z_1, \dotsc, z_d$ are equivalent to the score evaluations of the linear problem defined by $\vb_1, \dotsc, \vb_d$. Thus Theorem~\ref{theorem::linear_dim_bound} implies the maximum length of such a sequence can be of size at most $2\times 2 + 4 = 8 $. Finally the sub-additivity of the dissimilarity dimension (see Lemma~\ref{lemma::subaditivity_newdim}) implies,
\[
  \d_\rho(\Z_{N, \epsilon}, 1, \epsilon)  \leq \sum_{j=0}^1 \d_\rho(\Z_{j,N, \epsilon}, 1, \epsilon) \leq 16.
\qedhere
\]
\end{proof}

Combining the results of Lemmas~\ref{lemma::lower_bound_star_eluder_special_class} and~\ref{lemma::dissimilarity_dimension_upper_bound} finalizes the proof of Proposition~\ref{prop::separation_eluder_dissimilarity}.

\section{Multi-Armed Bandits}\label{appendix:mab}

In this section we explore the dissimilarity dimension of the $K$-armed bandit problem. In this setting the learner interacts with a set of $K$ arms and at every step of a sequential interaction pulls an arm $a_t \in [K]$ and receives a reward $r_t $ such that $\E[r_t] =  \mu_{i_t}$ where $\mu_{a_t}$ is the mean reward of arm $a_t$. For simplicity we will assume $\mu_a \in [0,1]$ for all $a \in [K]$ and that $|r_t | \leq 1$.

The $K$-armed bandit problem is an instance of structured bandits where $\A = [K]$ and $\F = [0,1]^{K}$. The dissimilarity dimension of the $K$-armed bandit problem satisfies,

\begin{proposition}
\label{prop::multi_armed_bandit}
Consider the action set $\A=[K]$ and the function class $\F =  [0,1]^{K}$ as defined above. Let $\alpha \in [0,1]$ and $\Z=\F\times\A$, $\score=\scoreBandits$ and $\epsilon \in (0,1/2)$. Then $\d_{\score}(\Z, \alpha, \epsilon)  \leq K$.
\end{proposition}

\begin{proof}
Let $c \leq  \alpha-\epsilon$ and $z_1, \dotsc, z_d \in \Z$ with $z_i = (f_i, a_i)$ be a maximal sequence such that, $\rho(z_i | z_i) \geq \alpha$ while
\begin{equation*}
    | \rho(z_i | z_j) - c| \leq \frac{\epsilon}{\sqrt{d}}.
\end{equation*}
For $i < j$. Substituting the definition of $\rho$, this implies $f_i(a_i) \geq \alpha$ for all $i \in [K]$ while $| f_j(a_i) - c| \leq \frac{\epsilon}{\sqrt{d}} $ for all $i <j$. By definition of $c$, if $d  \geq 2$
\begin{equation}\label{equation::upper_bounding_function_value}
    f_j(a_i ) \leq \alpha - \epsilon +  \frac{\epsilon}{\sqrt{d}} \leq \alpha - \left( 1-\frac{1}{\sqrt{2}}\right)\epsilon < \alpha - \frac{\epsilon}{4}, \quad \text{ for all } i <j.
\end{equation}
Let $I_i = \{ a_\ell \}_{\ell=1}^i$ be the set of actions up to index $i$ in the tuple sequence $z_1, \dotsc, z_i$. Equation~\ref{equation::upper_bounding_function_value} implies that $f_j(a) \leq \alpha-\frac{\epsilon}{4}$ for all $j > i$. Since $a_j$ satisfies $f_j(a_j) \geq \alpha > \alpha - \frac{\epsilon}{4}$ this implies $a_j \not\in I_i$. We conclude that  $a_i \neq a_j$ for all $i <j$. Since there are at most $K$ different arm values, this implies $d \leq K$.
\end{proof}

\subsection{Structured Bandits}

We will now explain in detail how what Algorithm~\ref{alg:MAIN1} reduces to in the structured bandits setting from Example~\ref{ex:bandits}. We write $z_i = (f_i, a_i)$ for all $i \in [T]$. The large evaluation set $\Z_\alpha$ can be reduced to the following set of functions,
\begin{equation*}
    \F_\alpha = \{ f \in \F \text{s t.} \max_{a \in \A} f(a) \geq \alpha \}.
\end{equation*}
Since the $\scoreBandits$ function $\scoreBandits( z_i | z ) $ is independent on $a$ for $z = (f,a)$, the least squares equation $\argmin_{z\in\Z_\alpha}
                      \sum_{i=1}^{t-1}
                      \BigParens{\score(z_i\given z) - \reward_i}^2$  reduces to,
\begin{equation*}
f_t = \argmin_{f \in \F_\alpha} \sum_{i=1}^{t-1} \left( f(a_i) - r_i\right)^2.
\end{equation*}
finally, to ensure the action query has a self evaluation of at least $\alpha$, we output $a_t = \argmax_{a \in \A} f_t(a)$.
\begin{algorithm}[t]
    \caption{Interactive Estimation for Structured Bandits}\label{alg:STRUCTURED}
    \begin{algorithmic}[1]
        \STATE \textbf{Input:} action set $\A$, function class $\F$, optimality level $\alpha$, number of steps $T$.
        Compute large-evaluation function-action set $\F_\alpha = \{ f \in \mathcal{F} \text{ s.t. } \max_{ a\in \A} f(a) \geq \alpha \}$ .
        \FOR{$t=1,\dotsc,T$}
     \STATE Compute regression function
     \begin{align*}
        f_t &= \argmin_{f \in \F_\alpha} \sum_{i=1}^{t-1} \left( f(a_i ) - r_i\right)^2.
\end{align*}
        \STATE Submit the query
               $
               a_t = \argmax_{a \in \A} f_t(a)$.
        \STATE Observe reward $\reward_t$.
        \ENDFOR
    \end{algorithmic}
\end{algorithm}

We will now explain in detail what the Optimistic Interactive Estimation Algorithm~\ref{alg:OPTIMISM} reduces to in the structured bandits setting from Example~\ref{ex:bandits}. We write $z_i = (f_i, a_i)$ for all $i \in [T]$.

 The least squares objective ($ z_t = \argmin_{z\in\Z_\alpha}
                      \sum_{i=1}^{t-1}
                      \BigParens{\score(z_i\given z) - \reward_i}^2$)
can be written as,
\begin{equation*}
\hat{f}_t = \argmin_{f \in \F} \sum_{i=1}^{t-1} \left( f(a_i ) - r_i\right)^2.
\end{equation*}
The action component of the $z$ element in this objective can be ignored since the $\scoreBandits$ evaluation function does not depend on it. The confidence ball $\Z_t = \biggBraces{z\in\Z:\:\smash[b]{\sum_{i=1}^{t-1}} \BigParens{\score(z_i\given z) - \score(z_i\given\hat{z}_t )}^2 \leq R}$ reduces to,
\begin{equation*}
\F_t = \biggBraces{f\in\F:\:\smash[b]{\sum_{i=1}^{t-1}} \BigParens{ f(a_i) - \hat{f}_t(a_i)}^2 \leq R}.
\end{equation*}
The query can be reduced to the action component of $z_t$,
\begin{equation*}
a_t = \argmax_{f,a \in \F_t \times \A} f(a).
\end{equation*}
Algorithm~\ref{alg:OPTIMISM_STRUCTURED} summarizes this reduction and corresponds to the standard optimistic least squares for sturctured bandit problems from~\cite{russo2013eluder}.

\begin{algorithm}[H]
    \caption{Optimistic Interactive Estimation for Structured Bandits}\label{alg:OPTIMISM_STRUCTURED}
    \begin{algorithmic}[1]
        \STATE \textbf{Input:} action set $\A$, function class $\F$, confidence-set radius $R$, number of steps $T$.
        \FOR{$t=1,\dotsc,T$}
     \STATE Compute confidence set
     \begin{align*}
        \hat{f}_t &= \argmin_{f \in \F} \sum_{i=1}^{t-1} \left( f(a_i ) - r_i\right)^2.
    \\
      \F_t &= \biggBraces{f\in\F:\:\smash[b]{\sum_{i=1}^{t-1}} \BigParens{ f(a_i) - \hat{f}_t(a_i)}^2 \leq R}.
\end{align*}
        \STATE Submit the query
               $
               a_t = \argmax_{f,a \in \F_t \times \A} f(a)$.
        \STATE Observe reward $\reward_t$.
        \ENDFOR
    \end{algorithmic}
\end{algorithm}

\end{document}